\documentclass{article}
\usepackage{bbm}
\usepackage{caption}
\usepackage{rotating}
\usepackage[maxfloats=100]{morefloats}
\usepackage{listings}
\usepackage{booktabs}
\usepackage{xcolor}
\usepackage[title]{appendix}
\usepackage{makecell} 

\usepackage{multirow}

\usepackage[symbol]{footmisc}

\usepackage{longtable}
\usepackage{subcaption}
\usepackage[numbers]{natbib}
\usepackage{geometry}

\usepackage{verbatim}
\usepackage{xcolor}
\usepackage{graphicx}
\usepackage{enumitem}
\usepackage{comment}
\usepackage{subfiles}

\usepackage{todonotes}
\usepackage{tikz-qtree,tikz-qtree-compat}
\usepackage{scrextend}
\usepackage{framed}
\usepackage{placeins}
\usepackage{setspace}
\usepackage[ruled,vlined]{algorithm2e}
\usepackage{mathtools,collcell,eqparbox}

\usepackage{multicol}
\usepackage{pgfplots}
\usepackage{tikz}
\usepackage{listings}
\usepackage{courier}
\usepackage{url}
\usepackage{color}
\usepackage[title]{appendix}
\usepackage{authblk}

\usepackage{amssymb}
\usepackage{amsthm}
\usepackage{dsfont}
\usepackage{amsmath}

\usepackage[pagebackref = true]{hyperref}
\hypersetup{
  colorlinks = true,
  linkcolor = teal,
  anchorcolor = teal,
  citecolor = teal,
  filecolor = teal,
  urlcolor = teal
  }

\pgfplotsset{compat=1.15}
\newcounter{BMatrix}

\newcommand{\setmaxwd}[1]{%
 \eqmakebox[BM-\theBMatrix][\BMalign]{$#1$}%
}
\MHInternalSyntaxOn

\MHInternalSyntaxOff
\makeatother

\newtheorem{theorem}{Theorem}
\newtheorem{assumption}{Assumption}

\newtheorem{lemma}{Lemma}

\theoremstyle{definition}
\newtheorem{definition}{Definition}

\theoremstyle{remark}
\newtheorem{remark}{Remark}
\newtheorem*{remark*}{Remark}

\newcommand{\inv}{^{-1}}

\newcommand{\1}{\mathds{1}}

\newcommand{\A}{\mathbb{A}}

\newcommand{\R}{\mathbb{R}}

\newcommand{\Z}{\mathbb{Z}}

\newcommand{\E}{\mathbb{E}}

\newcommand{\var}{\mathrm{Var}}

\newcommand{\cov}{\mathrm{Cov}}

\newcommand{\da}{\downarrow}
\newcommand{\ra}{\rightarrow}
\newcommand{\Ra}{\Rightarrow}

\newcommand{\la}{\leftarrow}

\newcommand{\cd}{\cdot}

\newcommand{\diag}{\mathrm{diag}}

\newcommand{\mrm}[1]{\mathrm{#1}}

\newcommand{\PiH}{\Pi_{\mathrm{H}}}

\newcommand{\cA}{\mathcal{A}}

\newcommand{\cF}{\mathcal{F}}

\newcommand{\cN}{\mathcal{N}}

\newcommand{\cP}{\mathcal{P}}

\newcommand{\cR}{\mathcal{R}}
\newcommand{\cS}{\mathcal{S}}
\newcommand{\cT}{\mathcal{T}}

\newcommand{\cX}{\mathcal{X}}

\newcommand{\cZ}{\mathcal{Z}}

\newcommand{\spnorm}[1]{\left|#1\right|_\mathrm{span}}
\newcommand{\set}[1]{\left\{{#1}\right\}}

\newcommand{\norm}[1]{\left\|#1\right\|}

\newcommand{\norminf}[1]{\left\|#1\right\|_{\infty}}

\newcommand{\abs}[1]{\left|#1\right|}
\newcommand{\sqbk}[1]{\left[ #1 \right]}

\newcommand{\crbk}[1]{\left( #1 \right)}

\numberwithin{equation}{section}

\title{Central Limit Theorem for Two-Time-Scale\\ Approximate Distributionally Robust RL}

\author{Shengbo Wang} 
\author{Zexi Zhang\thanks{The authors contributed equally to this work.}}
\affil{Daniel J. Epstein Department of Industrial and Systems Engineering, University of Southern California}

\date{May 2026}

\begin{document}
\maketitle

\begin{abstract}
Designing model-free algorithms for distributionally robust reinforcement learning (DRRL) poses fundamental challenges. The robust Bellman operator is nonlinear in the transition kernel, which makes one-sample Bellman updates biased, while the adversarial optimization underlying robustness makes robust evaluation computationally demanding. To address these difficulties, we consider the natural small-ambiguity regime under Kullback--Leibler ambiguity sets and propose an approximate DRRL framework based on a first-order expansion of the relevant robust functional. This yields an approximate robust Bellman equation that removes the adversarial optimization while remaining first-order accurate in the ambiguity radius. To learn the fixed point of this approximate equation, we propose Mean-Variance Stochastic Approximation (MVSA), a model-free algorithm that uses only one-sample updates. This is achieved via a lifted stochastic approximation dynamics and a two-time-scale design. We then prove convergence and a central limit theorem for MVSA: its main iterate satisfies a central limit theorem at the canonical $n^{-1/2}$ scale, with explicitly characterized asymptotic covariances. Finally, we validate our theoretical findings with a numerical experiment.
\end{abstract}

\section{Introduction}

We consider robust stochastic control for infinite-horizon discounted Markov decision processes (MDPs) under transition-model misspecification. In many applications, the nominal agent training environment is only an imperfect approximation of the true deployment environment, either because it is estimated from limited data or because the system is subject to perturbations and structural shifts that are not captured by the nominal model \citep{wang2024b}. A standard way to address this issue is to work with a robust Markov decision process \citep{Iyenga2005,Nilim2005,wiesemann2013robust}, in which decisions are evaluated against a family of plausible transition kernels, rather than a single nominal one. Specifically, the robust MDP optimizes
\[
V^\pi(s):=\inf_{P\in\cP} \E^{\pi,P}_s\sqbk{\sum_{t=0}^\infty \gamma^t r(X_t,A_t)}
\]
over admissible adapted policies $\pi\in\PiH$. Here $r$ is the reward function and $\gamma\in(0,1)$ is the discount factor. The robustness is induced by the adversarial infimum over the ambiguity set $\cP$.

In this paper, we consider approximate distributionally robust reinforcement learning (DRRL) under an SA-rectangular Kullback--Leibler (KL) divergence ambiguity set. In this setting, the robust optimal $Q$-function is characterized by the robust Bellman equation
\[
Q^*(s,a)=r(s,a)+\gamma \inf_{P\in \mathcal P_{s,a}} \E_{P}\left[\max_{b\in A} Q^*(X,b)\right],
\]
where $\cP_{s,a}=\set{P:D_{\mrm{KL}}(P\|P_0(\cdot|s,a))\le \delta}$. This Bellman equation has inspired model-free approaches such as Q-learning \citep{Liu2022} and policy gradient \citep{wang2022robust_PG}. The main obstacle, however, is that the robust functional is nonlinear in the transition kernel. When the nominal model is accessed only through samples, this nonlinearity makes one-sample Bellman estimates biased and thereby blocks a genuine one-sample model-free update.

This bottleneck already appears in the DRRL literature. Existing robust Q-learning and policy gradient methods typically control the Bellman bias by using multiple transition samples from a generative model at each state in every update--in fact, $\Omega(\epsilon^{-1})$ samples when $\epsilon$ is the target accuracy \citep{wang2022robust_PG,wang2024vrql}. As a result, the transition probabilities at each state are effectively estimated from multiple samples, so these methods are not model-free in the usual sense. Other approaches provide partial progress. \citet{wang2023b} improves the multilevel Monte Carlo based robust Q-learning method of \citet{zhou2021}. This achieves the canonical $n^{-1/2}$ rate in a generative model setting using only $O(1)$ samples per update. However, it require randomly many and usually more than one samples for each update. On the other hand, \citet{liang2023single} propose a three-time-scale algorithm in the single-trajectory setting and establish its convergence, but do not analyze its convergence rate. \citet{Yang2023} provide the first convergence rate result for a one-sample implementable algorithm in a generative model setting. However, the nested design in \citet{Yang2023} comes at the cost of a slower polynomial rate. Taken together, the existing literature does not reveal whether one can simultaneously achieve a genuinely one-sample implementable DRRL algorithm and the canonical \(n^{-1/2}\) convergence rate.

A secondary challenge is computational. DRRL algorithms often need to repeatedly evaluate the robust functional, which is not only computationally expensive, but also introduces a min-max layer into policy optimization that is known to create stability issues in other machine learning tasks \citep{mescheder2018GAN}. This substantially complicates the implementation, making DRRL formulations much less amenable to large-scale algorithmic development.

To address these difficulties, we first ask whether improving policy robustness truly requires an exact robust MDP formulation. This leads us to explore robust approximation in the small-ambiguity regime, reflecting the modeling assumption that the training environment can provide a reasonably informative approximation of the deployment environment \citep{wang2024vrql}, so that only robust refinements are needed. Specifically, our starting point is a first-order expansion of the KL-robust functional. This replaces the inner optimization in the robust Bellman equation by an explicit correction term, yielding an approximate robust Bellman equation that is much more computationally tractable while remaining first-order accurate in the ambiguity radius. In particular, its fixed point approximates the exact robust solution $Q^*$ with error of order $O(\delta)$.

The approximation alone, however, does not produce a one-sample algorithm, because the resulting Bellman operator is still nonlinear in the transition kernel. To handle this remaining nonlinearity, we introduce a lifted two-time-scale stochastic approximation (SA) scheme, Mean-Variance Stochastic Approximation (MVSA). At a high level, the lift moves the problematic nonlinear correction into auxiliary quantities that can be tracked recursively, while the time-scale separation stabilizes the interaction between these auxiliary recursions and the main Q-update. This is the key idea that makes single-sample updates possible in our approximate DRRL framework.

Our theoretical results separately study the approximation and learning errors. First, we show that the approximate robust Bellman operator is a contraction in the small-ambiguity regime and that its fixed point approximates the robust optimal $Q^*$ with $O(\delta)$ error. Second, we show that MVSA learns the fixed point of the stabilized approximate operator with the canonical stochastic-approximation rate $n^{-1/2}$. Finally, under a unique-greedy-action condition, we establish a central limit theorem with an explicit covariance characterization. Beyond convergence rates, this CLT provides a basis for convergence diagnostics and statistical inference for the learned robust value function.

\subsection{Literature Review}
\paragraph{Distributionally robust reinforcement learning.}
DRRL have been attracting significant research interests, leading to statistically efficient algorithmic designs in both model-based and model-free settings. On the model-based side, finite-sample guarantees under generative models have been established and sharpened to near-minimax-optimal rates for both SA- and S-rectangular ambiguity sets, under discounted and average-reward criteria \citep{panaganti2022,Yang2021,zhou2021,shi2022,Xu2023,shi2023,chen2025avg_rwd,li2025Srec}. By contrast, the statistical complexity of model-free methods is less well understood. Existing work has focused primarily on the SA-rectangular setting, where distributionally robust Q-learning and its variants have been developed with convergence and finite-sample guarantees \citep{Liu2022,wang2023b,Yang2023,wang2024vrql,liang2023single}. Complementing these value-based approaches, a parallel line of work studies policy-gradient methods for robust MDPs, establishing convergence under both rectangular \citep{wang2022robust_PG,Kumar2023,WangHo2023} and non-rectangular \citep{LiKuhn2025,Wangzha2025} ambiguity models. While this literature has largely focused on finite-time convergence and finite-sample complexity bounds, we develop a central limit theory that not only quantifies the convergence rate, but also enables statistical inference for learned robust value estimates.

\paragraph{Stochastic approximation.}
Stochastic approximation originates with \citet{Robbins1951}, and its convergence theory has since been developed \citep{Ljung1977,kushner2003,borkar2008}. Classical central limit theory for single-time-scale SA was established by \citet{Fabian1968} and \citet{Polyak1990}. In the multi-time-scale setting, linear SA recursions were shown to satisfy CLTs with decoupled convergence rates \citep{BORKAR1997,konda2004}, and these results were later extended to nonlinear recursions by \citet{MokkademPelletier2006}. More recently, \citet{chen2022finite_sample_SA} and \citet{chen2024lyapunov} developed general frameworks for finite-sample analysis of SA algorithms driven by martingale and Markovian noise.

\section{KL-Divergence Constrained Robust MDP} \label{sec:prelim}
Consider a SA-rectangular discounted robust MDP with finite state and action spaces $\cS$ and $\cA$, discount factor $\gamma \in (0,1)$, bounded reward function $r: \cS \times \cA \to \R$ with
\(
r_{\max} := \norm{r}_\infty,
\)
and nominal transition kernel $P_0(\cdot | s,a)$ on $\cS$. Throughout we write $\cZ := \cS \times \cA$ with  $d:=|\cZ|=|\cS||\cA|$, and identify vectors in $\R^d$ with functions on $\cZ$. For each state-action pair $z\in\cZ$, the ambiguity set is defined by a KL ball of radius $\delta \ge 0$:
\[
    \mathcal{P}_{z} := \set{ P(\cdot | z) : D_{\mathrm{KL}}(P \| P_0(\cdot\mid z)) \le \delta }.
\]
Then, the joint ambiguity set is $\cP:= \bigtimes_{z\in\cZ} \cP_z$. To avoid heavy notation, we suppress the dependence of $\cP_z$, $\cP$, and all subsequently defined objects on $\delta$.

The objective we consider is the maximization of the robust value functions over all admissible randomized history-dependent policies
\begin{equation}\label{eqn:Vstar}
V^*(s):=\sup_{\pi\in\PiH}\inf_{P\in\cP} \E^{\pi,P}_s\sqbk{\sum_{t=0}^\infty \gamma^t r(X_t,A_t)}
\end{equation}
where $\E_s^{\pi,P}$ is the expectation induced by policy $\pi$, kernel $P$, and initial state $X_0 = s$. 

In this context, dynamic programming theory for robust MDPs shows that optimality is attained within the class of stationary deterministic policies $\Pi$ \citep{wang2024b}, where each $\pi\in\Pi$ is identified with a function $\cS\ra\cA$. To characterize an optimal policy in $\Pi$, we define the robust optimal Q-function $Q^*:\cZ\to\R$ as the unique solution of the robust Bellman equation
\begin{equation}\label{eq:robust-bellman}
    Q^*(z)
    = r(z) + \gamma \inf_{P \in \mathcal{P}_{z}} \mathbb{E}_P\sqbk{\max_{b \in \cA} Q^*(X, b)}
    =: \mathcal{T}[Q^*](z),
\end{equation}
where $\mathcal{T}$ denotes robust Bellman operator. It is well known that $\mathcal{T}$ is a $\gamma$-contraction under $\|\cdot\|_\infty$ \citep{Iyenga2005}, and hence admits a unique fixed point $Q^*$. Then any greedy policy with respect to $Q^*$ is optimal; that is, any policy $\pi^*$ satisfying $\pi^*(s)\in\arg\max_{a\in\cA}Q^*(s,a)$ for all $s\in\cS$ achieves the optimal value in \eqref{eqn:Vstar}. Therefore, DRRL can be achieved by learning an estimator of $Q^*$, and an approximate robust optimal policy can then be obtained by taking a greedy policy. This property has inspired numerous algorithmic developments, including both model-based \citep{shi2022} and model-free \citep{wang2023b} designs.

As motivated in the introduction, the minimization in the robust Bellman operator \eqref{eq:robust-bellman} requires solving a constrained optimization problem, and is therefore a nonlinear functional of the input $Q$. As a result, sample-average-approximation-based estimators for the robust Bellman operator are biased. Thus, for model-free algorithms such as Q-learning \citep{wang2023b} and policy gradient \citep{wang2022robust_PG}, one would usually need $O(\text{poly}(\epsilon\inv))$ samples, where $\epsilon$ is the targeted accuracy, within each Q-update or gradient step to ensure convergence. This creates a bottleneck that substantially limits the design, convergence rate, and generalizability of model-free algorithms.

\section{Approximate Bellman Equation Using First-Order Expansion}\label{sec:expansion}

To address this algorithmic design difficulty, this paper proposes an alternative approach in the small-ambiguity-radius regime, that is, when $\delta$ is small, by developing an approximate DRRL framework based on a first-order expansion of the robust functional. Specifically, we leverage a polynomial asymptotic expansion of the KL-divergence-based worst-case expectation functional, and use its first-order term as a correction to the expectation under $P_0$; see \citet{Lam2016}.

To simplify notation, for any $U:\cZ\to\R$, define the max-over-actions map
\[
v[U](s) := \max_{a \in \cA} U(s,a)
\]
for all $s \in \cS$. Since $\cZ$ is finite, any $U:\cZ\ra\R$ is bounded and $\mathbb{E}_{P_0}[e^{\theta v[U](X)}] < \infty$ for all $\theta\in\R$. Then, Theorem~3.1 in \citet{Lam2016} implies that 
\begin{equation}
  \inf_{P:\, D_{\mathrm{KL}}(P \| P_0) \le \delta} \mathbb{E}_{P}[v[U](X)]
  = \mathbb{E}_{P_0}[v[U](X)]
    - \sqrt{2\,\operatorname{Var}_{P_0}(v[U](X))}\delta^{1/2}
    + O(\delta)
  \label{eq:kl-expansion}
\end{equation}
as $\delta \to 0$. We note that Theorem~3.1 in \citet{Lam2016} is stated under the additional assumption that $v[U](X)$ is not constant a.s.$P_0$. However, when $v[U](X)$ is constant a.s.$P_0$, \eqref{eq:kl-expansion} trivially holds.

\subsection{Approximate Bellman Equation and Error Bound}

For each $z \in \cZ$ and $\phi:\cX\ra\R$, we write
$\E_z[\phi(X)] := \E_{P_0(\cdot\mid z)}[\phi(X)]$,  $\var_z(\phi(X)) := \var_{P_0(\cdot\mid z)}(\phi(X))$, and $\cov_z(\phi(X),\psi(X)) := \cov_{P_0(\cdot\mid z)}(\phi(X),\psi(X)).$  Moreover, define following operators $\mu,\nu,\sigma:\R^d\to\R^d$ as
\begin{equation}\label{eqn:def_mu_nu_sigma}
\begin{aligned}
    \mu[U](z) &:= \E_z\sqbk{v[U](X)}
    = \E_z\sqbk{\max_{b \in \cA} U(X,b)}, \\
    \nu[U](z) &:= \E_z\sqbk{v[U](X)^2}
    = \E_z\sqbk{\crbk{\max_{b \in \cA} U(X,b)}^2}, \\
    \sigma[U](z) &:= \sqrt{\nu[U](z) - \mu[U](z)^2} = \sqrt{\var_z(v[U](X))},
\end{aligned}
\end{equation}
for every $z\in\cZ$. We then treat $\mu[U]$, $\nu[U]$, and $\sigma[U]$ as functions $\cZ\ra\R$. These abbreviations are used throughout the paper.

Applying \eqref{eq:kl-expansion} to \eqref{eq:robust-bellman} and use the notation defined above, we obtain $$Q^* = r + \gamma\,\mu[Q^*] - \gamma\sqrt{2\delta}\sigma[Q^*] + O(\delta).$$ Therefore, when $\delta$ is small, we expect $Q^*$ to be well approximated by the solution of the same equation with the $O(\delta)$ term removed. This motivates the following approximate robust Bellman operator, whose fixed point can serve as an approximation to $Q^*$.

\begin{definition}\label{def:approx-bellman}
For $\delta \ge 0$ and $U:\cZ\ra \R$, define
$\mathcal{R}[U] := r + \gamma\, \mu[U] -\gamma\sqrt{2\delta} {\sigma}[U]$.
\end{definition}

The approximate operator $\mathcal{R}$ inherits the Lipschitz property of $\cT$ under $\norm{\cd}_\infty$, with a slightly worse Lipschitz constant upper bound. In particular, under the following assumption, this Lipschitz property induces contraction.
\begin{assumption}\label{ass:contraction}
The parameters satisfy $L := \gamma(1 + \sqrt{2\delta}) < 1$.
\end{assumption}
This assumption holds for sufficiently small \(\delta\), which is exactly the regime where the first-order expansion \eqref{eq:kl-expansion} is accurate. Since \(\delta\) reflects confidence in the nominal model, this regime is practically the most relevant. Although the range of $\delta$ under Assumption \ref{ass:contraction} may seem restrictive, it is mainly needed for worst-case analysis. In many problems of interest, the following results remain valid for much larger \(\delta\), as illustrated in Section~\ref{sec:numerics}.

\begin{theorem}\label{thm:contraction}
Under Assumption~\ref{ass:contraction}, the operator $\mathcal{R}$ is an $L$-contraction in $\|\cdot\|_\infty$. Consequently, by the Banach fixed point theorem, it has a unique fixed point $U^{*}:\cZ\ra \R$ satisfying $U^{*} = \mathcal{R}[U^{*}]$.
\end{theorem}

An extended version of Theorem~\ref{thm:contraction} is proved in Appendix~\ref{section:proof:thm:contraction}. Since $U^*$ is well defined, it is natural to ask how close the fixed point $U^*$ of $\cR$ is to the true robust Q-function $Q^*$. The following theorem answers this question, and its proof is deferred to Appendix~\ref{section:proof:thm:approx-error}.

\begin{theorem}
\label{thm:approx-error}
Under Assumption~\ref{ass:contraction}. The fixed point $U^*$ of $\mathcal{R}$ and the robust optimal $Q$-function $Q^*$ satisfy
\[
  \|U^* - Q^*\|_\infty \leq \frac{\gamma}{1 - \gamma}  \delta \, |U^*|_{\mathrm{span}}\,,
\]
where $|U^*|_{\mathrm{span}} := \max_{z\in\cZ} U^*(z) - \min_{z\in\cZ} U^*(z).$
\end{theorem}
Therefore, Theorem~\ref{thm:approx-error} shows that the approximation error is $O(\delta)$ as $\delta\da 0$, which is consistent with the single-stage expansion in \eqref{eq:kl-expansion}.

\section{A Two-Time-Scale Algorithm}\label{sec:sa}

A first-order expansion makes the robust Bellman operator more tractable, but it does not by itself resolve the difficulty in model-free DRRL: The approximate operator is still a nonlinear functional of the transition kernel through the first-order correction term involving $\sigma[U]$. As a result, a naive one-sample plug-in update remains biased. Thus, to produce a valid one-sample model-free update, additional design is required.

To address this, our key idea is to lift the recursion to a higher-dimensional space by introducing auxiliary estimates for the first two conditional moments of $v[U]$. Specifically, for fixed $U$, \eqref{eqn:def_mu_nu_sigma} shows that $\sigma[U]$ is determined by $\mu[U]$ and $\nu[U]$, and each of these depends \textit{linearly} on the transition kernel. Hence, $\mu[U]$ and $\nu[U]$ can be estimated unbiasedly from a single next-state sample. Then, $\sigma[U]$ can then be estimated from the resulting estimates of $\mu[U]$ and $\nu[U]$.

This leads to our two-time-scale algorithm design: the auxiliary iterates $(m_n,g_n)$ track these moments on a fast time scale, while the main iterate $U_n$ is updated on a slower time scale using only the current moment estimates and no additional transition samples. In this way, the sampling noise is absorbed by the fast recursions, so the slow $U_n$-update is conditionally noise-free given the current iterates. This makes one-sample updating possible and yields the canonical $n^{-1/2}$ rate for the main iterate.
\subsection{Stabilized Operator}

Recall from Definition~\ref{def:approx-bellman} that the approximate operator $\cR$ involves the standard deviation $\sigma[U]$. When $\sigma[U^*](z)=0$ for some $z\in\cZ$, the Bellman operator loses differentiability, which can create stability issues for a naive SA implementation. To address this, we introduce an $\varepsilon$-perturbed standard deviation with a small safety parameter $\varepsilon\geq 0$:
${\sigma}_{\varepsilon}[U] := \sqrt{\nu[U]- \mu[U]^2 + \varepsilon}.$

Note that it is safe to choose $\varepsilon=0$ when the variance is nondegenerate, i.e., when $\sigma[U_\varepsilon^*](z)>0$ for all $z\in\cZ$. In practice, however, one may set a very small value such as $\varepsilon=10^{-6}$ to avoid degeneracy. 

We introduce the corresponding stabilized approximate Bellman operator
\begin{definition}
For $\delta,\varepsilon  \ge 0$ and $U:\cZ\ra\R$, define the stabilized operator 
$\mathcal{R}_\varepsilon[U]
:= r + \gamma \mu[U]
- \gamma \sqrt{2\delta}{\sigma}_{\varepsilon}[U]$.
\end{definition}
\begin{remark}\label{rmk:Reps_has_fixed_pt}
Similar to $\cR \equiv\cR_0$ in Definition \ref{def:approx-bellman}, we show in Theorem \ref{thm:contraction_eps} in the appendix that the operator $\cR_\varepsilon$ is an $L$-contraction in $\norm{\cd}_\infty$ under Assumption \ref{ass:contraction}. So, it has a unique fixed point denoted by $U_\varepsilon^*$. 
\end{remark}

\subsection{Mean-Variance Stochastic Approximation}

We propose a two-time-scale SA algorithm, Mean-Variance Stochastic Approximation (MVSA), to learn the fixed point $U_\varepsilon^*$ of $\cR_\varepsilon$. A clear two-time-scale separation is induced by the following requirements on the step sizes used in MVSA. 
\begin{assumption}
\label{ass:step-sizes}
We choose the step sizes $$\alpha_n = \frac{a}{n+a}\quad\text{and}\quad\beta_n = \frac{b}{(n+a)^\tau},$$ where $a > 1/(2(1 - L))$ is an integer with $L = \gamma(1 + \sqrt{2\delta})$, $0 < b\leq (1+a)^\tau$, and $\tau\in(1/2,1)$.
\end{assumption}

It should be noted that this specific form of the step size below is only for theoretical convenience. More generally, one may choose $\alpha_n\sim an^{-1}$ and $\beta_n\sim bn^{-\tau}$ with real constants $a>1/(2(1-L))$, $b>0$, and $\tau\in(1/2,1)$, and the convergence results should still hold.

Since $\tau<1$, it follows that $\alpha_n/\beta_n\to0$. This is known as the time-scale separation condition in the multi-time-scale algorithm analysis literature. Then, motivated by the two-time-scale model-free design outlined earlier, the MVSA algorithm, summarized in Algorithm \ref{alg:mvsa}, updates the fast and slow iterates as follows. 

\begin{algorithm}[th]
\caption{Mean-Variance Stochastic Approximation (MVSA)}
\label{alg:mvsa}
\DontPrintSemicolon

\KwIn{
Discount factor $\gamma \in (0,1)$; ambiguity radius $\delta \ge 0$; safety parameter $\varepsilon \geq 0$; step-size $\set{\alpha_n,\beta_n:n\ge 1}$; termination iteration $N$.
}
\textbf{Initialize: } $U\la 0;\ m \la 0;\ g\la 1$ in $\R^d$. \\
\For{$n=1,2,\dots,N$}{
Compute the perturbed standard deviation $\sigma \la \sqrt{g-m^2+\varepsilon}$.\\
Compute the new slow iterate
$U'\la U+\alpha_n(r+ \gamma m- \gamma\sqrt{2\delta}\,\sigma- U).$\\
    \ForEach{$z\in\cZ$}{
        Sample $X(z)\sim P_0(\cdot|z)$ and compute $Z(z) \la \max_{a\in\A}U(X(z),a)$. \\
        Update fast iterates
        $m(z)\la m(z) + \beta_n(Z(z)-m(z))$ and 
        $g(z)\la g(z) +\beta_n(Z(z)^2-g(z)).$
    }
Update the slow iterate $U \la U'.$
}
\Return{$U$}
\end{algorithm}

\noindent\textbf{Slow-time-scale update.}
For each $z\in\Z$, we first compute the estimate of the standard deviation
\[
\sigma_{n,\varepsilon}(z) := \sqrt{g_n(z) - m_n(z)^2 + \varepsilon}.
\]
Then, using the slow step size $\alpha_n$, we update the Q-function estimate:
\begin{equation}\label{eqn:slow_update}
U_{n+1}(z) := U_n(z) + \alpha_n\crbk{r(z) + \gamma\, m_n(z) - \gamma\sqrt{2\delta}\sigma_{n,\varepsilon}(z) - U_n(z)}.
\end{equation}
This update drives $U_n$ toward the fixed point of the first-order robust Bellman operator $\mathcal{R}$.

\noindent\textbf{Synchronous sampling.}
For every state-action pair $z \in \Z$, we independently draw a next-state sample $X'_{n+1}(z) \sim P_0(\cdot | z)$ from the nominal transition kernel. We then compute the value estimate
\begin{equation}\label{eqn:sample_Z}
Z_{n+1}(z) := \max_{b \in \A} U_n\crbk{X'_{n+1}(z),\, b},
\end{equation}
which represents a one-sample estimate of the next-state optimal value under the current value estimate $U_n$. 

\noindent\textbf{Fast-time-scale updates.}
Using the fast step size $\beta_n$ with $\tau \in (1/2,1)$, we update the running first and second moment estimates:
\begin{equation}\label{eqn:def_fast_update}
\begin{aligned}
m_{n+1}(z) &:= m_n(z) + \beta_n\crbk{Z_{n+1}(z) - m_n(z)}, \\
g_{n+1}(z) &:= g_n(z) + \beta_n\crbk{Z_{n+1}(z)^2 - g_n(z)},
\end{aligned}
\end{equation}
which track the conditional mean $\E_z[Z_{n+1}(z)]$ and the conditional second moment $\E_z[Z_{n+1}(z)^2]$, respectively.

\section{Convergence and Central Limit Theorem}
\label{sec:clt}
In this section, we establish the convergence and asymptotics of Algorithm~\ref{alg:mvsa}. We first show that, under the contraction property and time-scale separation condition, the last-iterate estimator $U_n$ returned by the algorithm converges to $U_\varepsilon^*$ almost surely. Moreover, under the additional assumption that the optimal policy is unique, $\sqrt{n}(U_n-U_\varepsilon^*)$ converges in distribution to a Gaussian limit. This not only implies that $\norm{U_n-U_\varepsilon^*}_\infty = O_P(n^{-1/2})$, thereby certifying the canonical convergence rate, but also enables statistical inference for optimal decisions based on Gaussian quantiles, providing additional leverage for uncertainty quantification in policy optimization.

Before stating the convergence results, we first characterize the limiting objects to which the iterates of Algorithm~\ref{alg:mvsa} converge. In particular, due to the safety parameter $\varepsilon$, we should expect that the algorithm converges to the unique fixed point $U^*_\varepsilon$ of $\cR_\epsilon$ (c.f. Remark \ref{rmk:Reps_has_fixed_pt}) and the corresponding moments. Specifically, define 
\begin{equation}\label{eq:safe_parameter}
m_{\varepsilon}^{*} := \mu[U_{\varepsilon}^{*}],\quad
g_{\varepsilon}^{*} := \nu[U_{\varepsilon}^{*}],\quad \text{and}\quad
\sigma_{\varepsilon}^{*} := \sqrt{\nu[U_{\varepsilon}^{*}] - \mu[U_{\varepsilon}^{*}]^2+\varepsilon}.
\end{equation}

Moreover, for the CLT, we require an additional structural assumption on the fixed point $U_{\varepsilon}^{*}$.

\begin{assumption}
\label{ass:unique-greedy}
If one chooses $\varepsilon = 0$, assume in addition that $\sigma[U_{\varepsilon}^{*}](z) > 0$ for all $z \in \cZ$. 
Moreover, for each $s \in \mathcal{S}$, there exists a unique greedy action $a^*(s) = \arg\max_{a \in \mathcal{A}}\, U_{\varepsilon}^{*}(s,a).$
\end{assumption}

\begin{remark}\label{rmk:unique_greedy_clt}
This assumption ensures that the Bellman operator is locally smooth near the fixed point, which is needed for the linearization argument underlying the CLT. It is well known that, when the optimal action is non-unique, one should not, in general, expect a Gaussian limiting distribution, since the maximum of the coordinates of a Gaussian vector is typically not Gaussian. In this sense, the assumption is close to necessary for obtaining a Gaussian CLT.
\end{remark}

\begin{theorem}
\label{thm:convergence-clt}
Suppose Assumptions \ref{ass:contraction} and \ref{ass:step-sizes} are in force. Then, $(U_n, m_n, g_n) \to (U_{\varepsilon}^{*},\, m_{\varepsilon}^{*},\, g_{\varepsilon}^{*})$ almost surely as $n \to \infty$, with $$\norminf{U^{*} - U_{\varepsilon}^{*}}\leq \frac{\gamma\sqrt{2\delta\epsilon}}{1-L}.$$

If in addition Assumption \ref{ass:unique-greedy} holds, then
\[
  \begin{bmatrix} \sqrt{n}\,(U_n - U_{\varepsilon}^{*}) \\[4pt] n^{\tau/2}\,\crbk{(m_n, g_n) - (m_{\varepsilon}^{*}, g_{\varepsilon}^{*})} \end{bmatrix}
  \Ra
  \cN\left( 0, \begin{bmatrix} a\Sigma_U & 0 \\ 0 & b\Sigma_{(m,g)} \end{bmatrix} \right),
\]
where the covariance matrices $\Sigma_U$ and $\Sigma_{(m,g)}$ are defined in Section \ref{section:covariances}.
\end{theorem}

\subsection{Asymptotic Covariance}\label{section:covariances}
In this section, we specify the covariance matrices associated with the CLT result in Theorem \ref{thm:convergence-clt}. Let $\{X'(z):z \in \cZ\}$ be sampled independently with $X'(z)\sim P_0(\cd|z)$. Define
$Z_{\varepsilon}^*(z) := v[U_{\varepsilon}^*](X'(z)).$ We introduce the following moment quantities: 
\begin{equation}\label{eq:VCW}
\begin{aligned}
V(z) &:= \var\crbk{Z_{\varepsilon}^*(z)} = g_{\varepsilon}^*(z) - m_{\varepsilon}^*(z)^2, \\
C(z) &:= \cov\crbk{Z_{\varepsilon}^*(z),\, Z_{\varepsilon}^*(z)^2} = \E[Z_{\varepsilon}^*(z)^3] - m_{\varepsilon}^*(z)\, g_{\varepsilon}^*(z), \\
W(z) &:= \var\crbk{Z_{\varepsilon}^*(z)^2} = \E[Z_{\varepsilon}^*(z)^4] - g_{\varepsilon}^*(z)^2.
\end{aligned}
\end{equation}

\paragraph{Covariance for the fast iterates.} We define the covariance matrix for the fast iterates as
\[
\Sigma_{(m,g)}  := \frac{1}{2}\begin{bmatrix} \Gamma_{mm} & \Gamma_{mg} \\ \Gamma_{gm} & \Gamma_{gg} \end{bmatrix},
\]
with block entries indexed by $(z, z') \in \Z \times \Z$ given by
\begin{equation}\label{eq:gamma22}
\Gamma_{mm} = \mathrm{diag}\crbk{V}, \quad
\Gamma_{mg} = \mathrm{diag}\crbk{C}, \quad
\Gamma_{gg} = \mathrm{diag}\crbk{W},
\end{equation}
and $\Gamma_{gm} = \Gamma_{mg}^\top = \Gamma_{mg}$. 

\paragraph{Covariance for the slow iterate.} 
Note that under the additional Assumption \ref{ass:unique-greedy}, the greedy action $a^*(s)$ is uniquely defined for all $s\in \cS$. Then, define for $z,z'\in\cZ$, $Q_0(z,z') := P_0(s' \mid z)\, \1\{a' = a^*(s')\}$
and 
\begin{equation}\label{eq:H}
H(z,z') := -I(z,z')  + \crbk{1 + \frac{\sqrt{2\delta}m_{\varepsilon}^*(z)}{\sigma_{\varepsilon}^*(z)} -\frac{\sqrt{2\delta}U_{\varepsilon}^*(z')}{\,\sigma_{\varepsilon}^*(z)}}\gamma Q_0(z,z').
\end{equation}
where $I$ is the identity matrix. Then, the covariance matrix $\Sigma_U$ for the slow iterate is
\begin{equation}\label{eq:sigmaU}
\Sigma_U := \int_0^\infty \exp\sqbk{\crbk{H + \frac{ I}{2a}} t} \Gamma_U \exp\sqbk{\crbk{H + \frac{I}{2a} }^\top t} \mathrm{d}t.
\end{equation}
It is the unique solution to the Lyapunov equation
$\crbk{H +  I/(2a)}\, \Sigma_U + \Sigma_U \crbk{H +  I/(2a)}^\top = -\Gamma_U.$
Here $\Gamma_U$ is a $d \times d$ diagonal matrix 
\begin{equation}\label{eq:gammaU}
\Gamma_U = \diag\sqbk{\crbk{1 + \frac{\sqrt{2\delta}m_{\varepsilon}^*}{\sigma_{\varepsilon}^*}}^2 \gamma^2V  - \crbk{1 + \frac{\sqrt{2\delta}m_{\varepsilon}^*}{\sigma_{\varepsilon}^*}}\crbk{\frac{\sqrt{2\delta}}{\,\sigma_{\varepsilon}^*}} \gamma ^2C + \crbk{\frac{\sqrt{2\delta}}{2\,\sigma_{\varepsilon}^*}}^2 \gamma^2W}.
\end{equation}

\section{Numerical Experiments}\label{sec:numerics}

We evaluate the proposed MVSA algorithm on a classical discrete-time inventory control problem with backlogging under stochastic demand. Our experiments focus on assessing both the approximation quality of the approximate operator and the convergence behavior of the learning algorithm. 

\subsection*{A Robust Inventory Management Model}

We formulate a back-order inventory model as a controlled Markov process $\{(X_t, A_t) : t \ge 0\}$ with finite state and action spaces. The state variable $X_t$ represents the inventory level at the beginning of period $t$, taking values in $\cS := \{-B, -B+1, \dots, 0, \dots, I\},$ where $I>0$ is the inventory capacity and $B>0$ is the maximum allowable backlog. At each period, the inventory manager selects an order quantity from $\mathcal{A} := \{0,1,\dots,O\},$ where $O>0$ is the maximum order size.

At time $t$, an order quantity $A_t$ is selected. Due to the inventory capacity constraint, the effective order is $\widetilde{A}_t = \min\{A_t,\, I - X_t\}.$ The demand process $\{D_t\}$ is assumed to be i.i.d. with support $\{0,1,\dots,D_{\max}\}$ and nominal distribution $P_D$. Given $(X_t, A_t)$ and demand realization $D_t$, the system evolves according to 
\begin{equation}
X_{t+1} = \max\{X_t + \widetilde{A}_t - D_t, -B\}.
\label{eq:inv-dynamics}
\end{equation}
On the other hand, the instantaneous reward collected at time $t$ is given by
\begin{equation}
r(X_t, A_t, X_{t+1})
=
p(X_t - X_{t+1} + \widetilde{A}_t)
+ b \min(X_{t+1}, 0)
- h \max(X_{t+1}, 0)
- c \widetilde{A}_t,
\label{eq:inv-reward}
\end{equation}
where $p,c,h,b > 0$ are the unit sales price, ordering cost, holding cost, and backlog penalty, respectively.

In our numerical experiments, we set
\(
I = 10\), \(B = 5\), \( O = 5\), \( p = 3\), \( c = 2\), \( h = 0.2\), \(b = 3\),
and use the nominal demand distribution
\(
P_D = [0.1,\, 0.2,\, 0.3,\, 0.3,\, 0.1],
\)
supported on $\{0,1,2,3,4\}$.

\subsection*{Approximation Error}

\begin{figure}[t]
    \centering
    \includegraphics[width=0.7\textwidth]{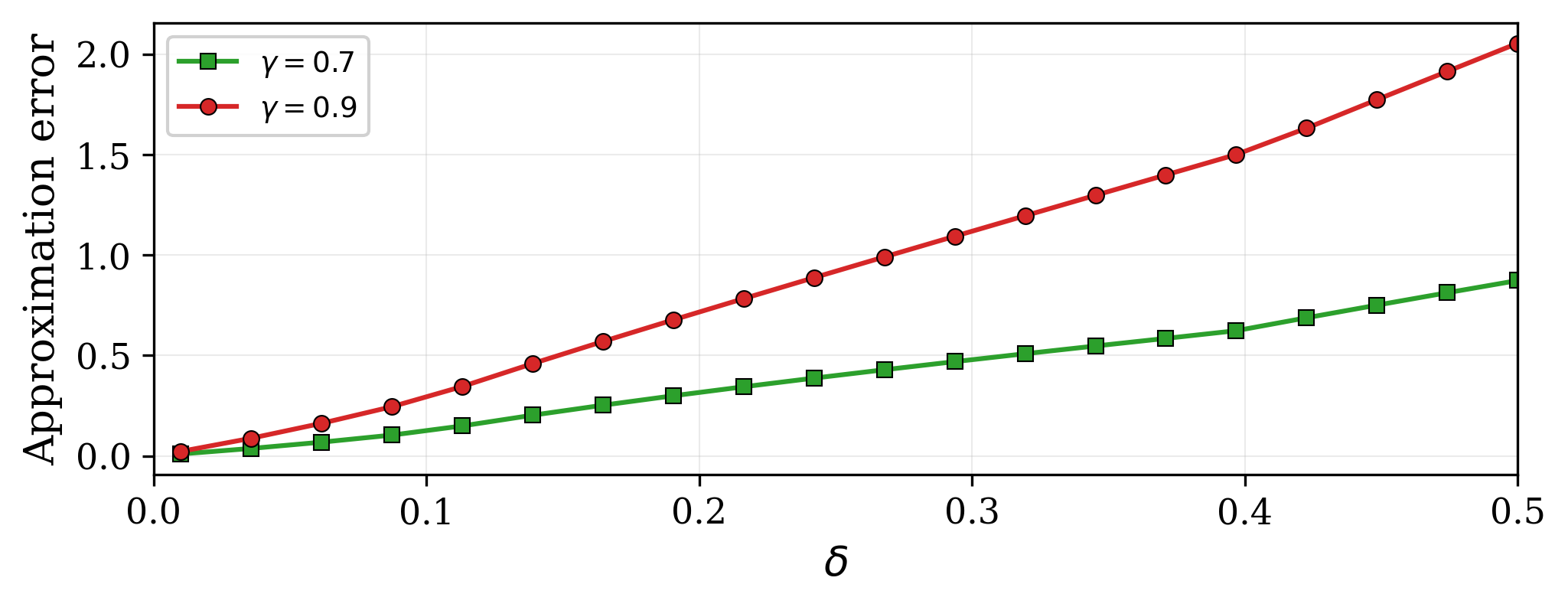}
    \caption{Approximation error $\|U^* - Q^*\|_\infty$ as a function of the ambiguity radius $\delta$.}
    \label{fig:approximation}
\end{figure}

Figure~\ref{fig:approximation} reports the approximation error $\|U^*-Q^*\|_\infty$ as a function of the ambiguity radius $\delta$ uniformly spaced in $[0.01, 0.5]$, for discount factors $\gamma\in\{0.7,0.9\}$. The plot shows that the error increases monotonically with $\delta$. We also observe that both curves grow approximately linearly in $\delta$, suggesting that the approximate Bellman operator $\cR$ captures the desired first-order correction. Moreover, the error curve becomes steeper for larger discount factors. Both phenomena are reflected in the theoretical upper bound in Theorem~\ref{thm:approx-error} even if Assumption \ref{ass:contraction} is violated. 

\subsection*{Convergence and Asymptotic Normality} 
\begin{figure}[t]
    \centering
    \includegraphics[width=\textwidth]{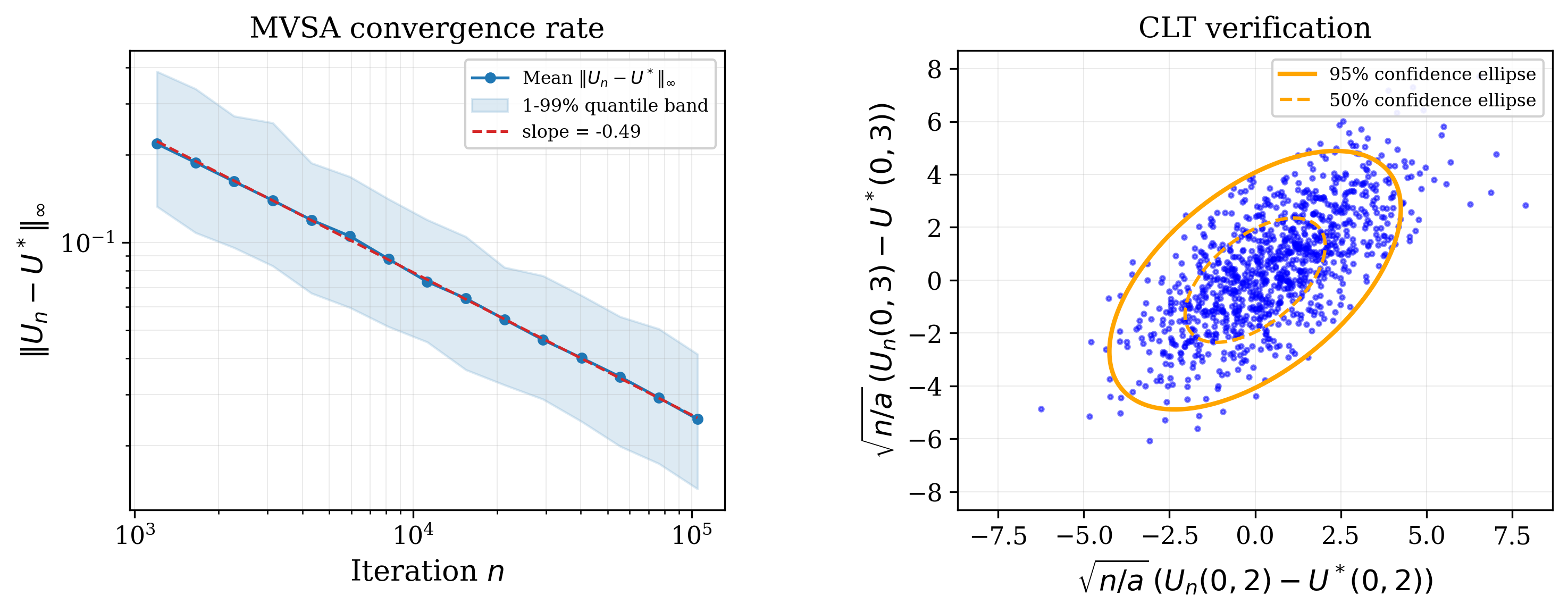}
    \caption{(Left) Estimation error $\|U_n - U^*\|_\infty$ on a log-log scale. (Right) 
Empirical distribution of the scaled error $\sqrt{n/a}[U_n(z) - U^*(z)]$ at $z = (0,2)$ and $(0,3)$.}
    \label{fig:convergence-clt}
\end{figure}
The left panel of Figure~\ref{fig:convergence-clt} visualizes the estimation error $\|U_n-U^*\|_\infty$ as a function of the iteration number $n$ on a log--log scale, averaged over independent runs with the 1--99\% quantile band shaded. Again, we use $\delta=0.1$ and $\gamma=0.7$ in the same instance. The error decreases steadily over iterations, with a fitted slope of $-0.49$, indicating a convergence rate of order $n^{-1/2}$. The stable quantile band confirms that this convergence trend is consistent across runs. 

To validate the asymptotic results, the right panel of Figure~\ref{fig:convergence-clt} plots $1{,}000$ independent realizations of the rescaled error $\sqrt{n/a}[U_n(z)-U^*(z)]$ at two representative state-action pairs $z = (0,2)$ and $(0,3)$ with $n=20{,}000$. The sample cloud is approximately elliptical, and the fitted confidence ellipses align well with the empirical distribution: the $95\%$ theoretical ellipse computed from Theorem~\ref{thm:convergence-clt} attains an empirical coverage of $92.8\%$.

\section{Conclusion}
In this paper, we propose an approximate DRRL framework for the small-ambiguity regime, leading to a one-sample implementable model-free algorithm that attains a CLT under the canonical $n^{-1/2}$ rate. Our approach is motivated by the observation that, in this regime, exact robust MDP formulations, while theoretically informative, can create substantial algorithmic and computational difficulties. By combining a first-order expansion with a lifted two-time-scale SA scheme, we obtain a more scalable learning procedure, MVSA, that avoids direct evaluation of the robust functional while still preserving the desired robustness at first order. We view this as a novel alternative approach toward scalable robust RL. 

However, as an initial attempt, several limitations remain to be addressed. First, the proposed sufficient condition established for the existence, uniqueness, and approximation error bound for the fixed point appears conservative. Figure \ref{fig:approximation} suggests that Theorems \ref{thm:contraction} and \ref{thm:approx-error} should hold well beyond Assumption \ref{ass:contraction}, but a sharper characterization of the approximation theory remains open. Second, while the CLT provides an asymptotic basis for statistical inference, making this practically usable requires covariance estimates that are compatible with the MVSA algorithmic structure. Third, our analysis is developed in the generative model setting. Extending comparable convergence-rate and asymptotic results to the single-trajectory setting is technically nontrivial because of the two-time-scale architecture, but seems plausible and is an important direction for future work.

\bibliographystyle{apalike}
\bibliography{references}

\newpage
\appendix
\appendixpage
\section{Proof of An Extension to Theorem \ref{thm:contraction}}\label{section:proof:thm:contraction}

In this section, we show an extended version of Theorem \ref{thm:contraction}. Fix $\varepsilon \geq 0$, we first recall the definition of the stabilized approximate Bellman operator: 
\[
\cR_{\varepsilon}[U](z)
=
r(z)+\gamma \mu[U](z)-\gamma\sqrt{2\delta}\,
\sqrt{\nu[U](z)-\mu[U](z)^2+\varepsilon}.
\]
for all $z\in Z$.

\begin{theorem}\label{thm:contraction_eps}

Suppose Assumption~\ref{ass:contraction} holds, the operator $\cR_{\varepsilon}[U]$ is an $L$-contraction in $\|\cdot\|_\infty$. Hence it has an unique fixed point $U_\varepsilon^{*} = \mathcal{R_\varepsilon}[U_\varepsilon^{*}]$ where $U_\varepsilon^*:Z\ra\R$.
\end{theorem}

\begin{proof}[Proof of Theorem \ref{thm:contraction_eps}]

Since $Z = S\times A$ is finite, the space of real-valued functions on $Z$ is a Banach space under the infinity norm $\norminf{\cd}$. Therefore, if we can show that $\cR_\varepsilon$ is a contraction in $\norminf{\cd}$, then the Banach fixed point theorem implies the existence and uniqueness of a fixed-point $U^*_\varepsilon = \cR_\varepsilon[U^*_\varepsilon]$, and hence Theorem \ref{thm:contraction_eps} follows. Since $L < 1$ under Assumption \ref{ass:contraction}, to this end, it suffices to show that for any
$U_1,U_2:Z\ra \R$ and any $z\in Z$,
\begin{equation}\label{eqn:to_prove_R_eps_contract}
|\cR_\varepsilon[U_1](z)-\cR_\varepsilon[U_2](z)|
\le L \|U_1-U_2\|_\infty. 
\end{equation}

To show \eqref{eqn:to_prove_R_eps_contract}, we first decompose
\begin{equation}\label{eqn:tu_R_R_two_term_bd}
\begin{aligned}
|\cR_\varepsilon[U_1](z)-\cR_\varepsilon[U_2](z)|
&\le
\gamma |\mu[U_1](z)-\mu[U_2](z)|\\
&\quad
+\gamma\sqrt{2\delta}\,
\Big|
\sqrt{\nu[U_1](z)-\mu[U_1](z)^2+\varepsilon}
-
\sqrt{\nu[U_2](z)-\mu[U_2](z)^2+\varepsilon}
\Big|. 
\end{aligned}
\end{equation}
We bound each term on the right-hand side separately.

By the definition of $\mu[\cdot]$, we bound the first term by 
\begin{align}\label{eq:mean_bound}
|\mu[U_1](z)-\mu[U_2](z)|
\le
\|v[U_1]-v[U_2]\|_\infty
\le
\|U_1-U_2\|_\infty,
\end{align}
where the last inequality follows from
\begin{align*}
\|v[U_1]-v[U_2]\|_\infty
&=
\max_{s\in S} |v[U_1](s)-v[U_2](s)|\\
&\le
\max_{s\in S}\max_{a\in A}|U_1(s,a)-U_2(s,a)|\\
&=
\|U_1-U_2\|_\infty.   
\end{align*}

We next bound the second term in \eqref{eqn:tu_R_R_two_term_bd} that corresponds to the variance component. Note that for $a,b\geq 0$, $a-b = (\sqrt{a}+\sqrt{b})(\sqrt{a}-\sqrt{b})$. So, for any \(x\geq y\ge 0\), we have
\[
\left|\sqrt{x+\varepsilon}-\sqrt{y+\varepsilon}\right|
=
\frac{x-y}{\sqrt{x+\varepsilon}+\sqrt{y+\varepsilon}},
\quad\text{and}\quad
\left|\sqrt{x}-\sqrt{y}\right|
=
\frac{x-y}{\sqrt{x}+\sqrt{y}}.
\]
Since \(\varepsilon\ge 0\), we have $\sqrt{x+\varepsilon}+\sqrt{y+\varepsilon}\ge \sqrt{x}+\sqrt{y}$,
and therefore
$\left|\sqrt{x+\varepsilon}-\sqrt{y+\varepsilon}\right|
\le
\left|\sqrt{x}-\sqrt{y}\right|.$ The case \(y\ge x\) is identical by symmetry. Applying this identity, we get
\begin{equation}\label{eqn:tu_std_eps_sigma_bd}
\begin{aligned}
&\Big|
\sqrt{\nu[U_1](z)-\mu[U_1](z)^2+\varepsilon}
-
\sqrt{\nu[U_2](z)-\mu[U_2](z)^2+\varepsilon}
\Big|
\\
&\quad \le
\left|
\sqrt{\nu[U_1](z) - \mu[U_1](z)^2}
-
\sqrt{\nu[U_2](z) - \mu[U_2](z)^2}
\right|\\
&\quad =\abs{\sigma[U_1](z) - \sigma[U_2](z)}. 
\end{aligned}
\end{equation}

We note that 
\begin{align*}(\sigma[U_1](z) - \sigma[U_2](z))^2 &= \var_z(v[U_1](X)) + \var_z(v[U_2](X)) - 2\sqrt{\var_z(v[U_1](X))\var_z(v[U_2](X))} \\
&\stackrel{(i)}{\leq} \var_z(v[U_1](X)) + \var_z(v[U_2](X)) -2\cov_z (v[U_1](X),v[U_2](X))\\
&= \var_z (v[U_1](X)-v[U_2](X))\\
&\leq \norminf{v[U_1]-v[U_2]}^2\\
&\leq  \norminf{U_1-U_2}^2
\end{align*}
where $(i)$ follows from the correlation relationship $\mathrm{Corr}(X,Y):=\cov(X,Y)/\sqrt{\var(X)\var(Y)} \leq 1$. 
Therefore, going back to \eqref{eqn:tu_std_eps_sigma_bd}, we have shown that
\begin{equation}\label{eq:variance_bound}
    \Big|
\sqrt{\nu[U_1](z)-\mu[U_1](z)^2+\varepsilon}
-
\sqrt{\nu[U_2](z)-\mu[U_2](z)^2+\varepsilon}
\Big|\leq  \norminf{U_1-U_2}
\end{equation}

Finally, to conclude Theorem \ref{thm:contraction_eps}, we combine \eqref{eqn:tu_R_R_two_term_bd}, \eqref{eq:mean_bound}, and \eqref{eq:variance_bound}, we obtain
\[
|\cR_\varepsilon[U_1](z)-\cR_\varepsilon[U_2](z)|
\le
\gamma\|U_1-U_2\|_\infty
+
\gamma\sqrt{2\delta} \|U_1-U_2\|_\infty = L \|U_1-U_2\|_\infty.
\]
This shows \eqref{eqn:to_prove_R_eps_contract}, and hence completes the proof. 
\end{proof}

\section{Proof of Theorem \ref{thm:approx-error}}\label{section:proof:thm:approx-error}
\begin{proof}
    
By the fixed-point identities $U^{*} = \mathcal{R}[U^{*}]$ and $Q^* = \mathcal{T}[Q^*]$ and the triangle inequality, we have
\begin{equation}\label{eq:triangle}
    \|U^{*} - Q^*\|_\infty
    = \|\mathcal{R}[U^{*}] - \mathcal{T}[Q^*]\|_\infty
    \leq \|\mathcal{R}[U^{*}] - \mathcal{T}[U^{*}]\|_\infty
    + \|\mathcal{T}[U^{*}] - \mathcal{T}[Q^*]\|_\infty.
\end{equation}
Since $\mathcal{T}$ is a $\gamma$-contraction under $\|\cdot\|_\infty$, we have $\|\mathcal{T}[U^{*}] - \mathcal{T}[Q^*]\|_\infty \leq \gamma\|U^{*} - Q^*\|_\infty$. Apply this to \eqref{eq:triangle} and rearrange the inequalities gives
\begin{equation}\label{eq:reduction}
    \|U^{*} - Q^*\|_\infty \leq \frac{1}{1-\gamma}\|\mathcal{R}[U^{*}] - \mathcal{T}[U^{*}]\|_\infty.
\end{equation}
By the definitions, the one-step difference at each $z\in\cZ$ is given by
\[\begin{aligned}
    \mathcal{R}[U^{*}](z) - \mathcal{T}[U^{*}](z)
    &= \gamma\crbk{\E_z[v[U^{*}](X)]
    - \sqrt{2\delta}\sigma[U^{*}] (z)
    - \inf_{P \in \mathcal{P}_{z}} \mathbb{E}_P\sqbk{v[U^{*}](X)}}
\end{aligned}
\]
Combining this with \eqref{eq:reduction}, we see that to prove Theorem \ref{thm:approx-error}, it suffices to show that for every $z \in \cZ$,
\begin{equation}\label{eq:pointwise_goal}
    \Bigl|\inf_{P \in \mathcal{P}_{z}} \mathbb{E}_P\sqbk{v[U^{*}](X)}
    - \crbk{\E_z[v[U^{*}](X)] - \sqrt{2\delta}\sigma[U^{*}] (z)}\Bigr|
    \leq\,\delta \,|U^{*}|_{\mathrm{span}}.
\end{equation}

The remaining of the proof aims to show \eqref{eq:pointwise_goal}. To this end, we derive upper and lower bounds for $\inf_{P \in \mathcal{P}_{z}} \mathbb{E}_P[v[U^{*}](X)]$ with a matching first-order expansion. 

\medskip
\noindent\textbf{Upper bound.}
To upper bound $\inf_{P \in \mathcal{P}_{z}} \mathbb{E}_P[v[U^{*}](X)]$, we can exhibit a feasible distribution $Q \in \mathcal{P}_{z}$, as \begin{equation}\label{eq:feasible}
    \inf_{P \in \mathcal{P}_{z}} \mathbb{E}_P[v[U^{*}](X)] \leq \mathbb{E}_Q[v[U^{*}](X)], \qquad \forall\, Q \in \mathcal{P}_{z}.
\end{equation}

Define the centered random variable
\[
W := v[U^*](X) - \mu[U^*](z) = v[U^*](X) - \E_z [v[U^*](X)].
\]
Note that $\mathbb{E}_z[W] = 0$ and \begin{equation} \label{eqn:tu_EW2_var}\mathbb{E}_z[W^2]= \var_z(v[U^*](X)) =\sigma[U^*](z)^2.
\end{equation}

With $W$, we define a $\eta$-parameterized family of candidate likelihood ratios
\[
L(\eta) := 1 - \eta W.
\]
Note that $|W| \le |U^*|_{\mathrm{span}}$. So, to use $L(\eta)$ as a likelihood ratio, we choose $\eta < 1 / |U^*|_{\mathrm{span}}$. Then, $L(\eta) > 0$ and
$\mathbb{E}_z[L(\eta)] = 1$; i.e. $L(\eta)$ defines a valid likelihood ratio. Hence, we can define the distribution $Q_\eta$ on $S$ by a change of measure
\[
\mathbb{E}_{Q_\eta}[\cdot] := \mathbb{E}_z[\cdot \, L(\eta)].
\]

Next, we choose $\eta$ so that $Q_\eta$ is feasible; i.e. $Q_\eta\in\cP_z$. To this end, we bound the KL divergence between $Q_\eta$ and $P_0(\cdot | z)$. Define for $w > 0 $
\[
g(w) := (1 - w)\log(1 - w) + w.
\]
Since $g(0) = g'(0) = 0$, a second-order expansion gives
\[
g(w) = g''(\xi)w^2 = \frac{w^2}{2(1 - \xi)} \le \frac{w^2}{2(1 - |w|)},
\]
where $\xi\in[0,w]$ is given by the mean value theorem. Using the $\phi$-divergence representation of KL, we obtain
\begin{align*}
D_{\mathrm{KL}}(Q_\eta \,\|\, P_0(\cdot | z)) &= \mathbb{E}_z[g(\eta W)]\\
&\le \mathbb{E}_z\left[\frac{\eta^2 W^2}{2(1 - \eta |W|)}\right]\\
&\leq     \frac{\eta^2 \sigma[U^*](z)^2}{2(1 - \eta |U^*|_{\mathrm{span}})}
\end{align*}
where the last inequality follows from $\eta < 1/\spnorm{U^*}$ and \eqref{eqn:tu_EW2_var}. Then, we enforce the feasibility of $Q_\eta$ by choosing $\eta$ that satisfy
\[
D_{\mathrm{KL}}(Q_\eta \,\|\, P_0(\cdot | z))
\le \frac{\eta^2 \sigma[U^*](z)^2}{2(1 - \eta |U^*|_{\mathrm{span}})}\le \delta. 
\]

If $\sigma[U^*](z) = 0$, the feasibility is trivial: any choice of $\eta < 1/\spnorm{U^*}$ would work. Otherwise, setting the KL constraint equal to $\delta$, we solve the quadratic equation, yielding
\[
\eta_\delta 
= \frac{-2\delta |U^*|_{\mathrm{span}} + \sqrt{4\delta^2 |U^*|_{\mathrm{span}}^2 + 8\sigma[U^*](z)^2 \delta}}{2\sigma[U^*](z)^2} >0.
\]
which satiafies $\eta_\delta^2\sigma^2 = 2(1-\eta_\delta \spnorm{U^*}) \delta$. Moreover, the constraint $1 - \eta_\delta |U^*|_{\mathrm{span}} > 0$ is automatically satisfied for such $\eta_\delta$ which makes $L(\eta_\delta)$ a valid likelihood ratio. Therefore, $Q_{\eta_\delta}\in D_{\mathrm{KL}}(Q_\eta \,\|\, P_0(\cdot | z)) = \cP_z$. 

Then, by \eqref{eq:feasible}, to achieve an upper bound on the infimum, we can exhibit an upper bound on $\E_{Q_{\eta_\delta}} [ v[U^*](X)]$. To this end, we first establish the following lower bound on $\eta_\delta$: 
\[
\eta_\delta 
\ge \frac{-2\delta |U|_{\mathrm{span}} + \sqrt{ 8\sigma[U^*](z)^2 \delta}}{2\sigma[U^*](z)^2}
\ge \frac{\sqrt{2\delta}}{\sigma[U^*](z)} - \frac{\delta |U|_{\mathrm{span}}}{\sigma[U^*](z)^2}.
\]
Then, under $Q_{\eta_\delta} \in \mathcal{P}_{z}$, we have
\begin{equation}\label{eq:Q_etabound}
\begin{aligned}
\mathbb{E}_{Q_{\eta_\delta}}[v[U^*(X)]]
&= \mathbb{E}_z[v[U^*(X)]] - \eta_\delta \, \mathbb{E}_z\left[v[U^*(X)] W\right] \\
&= \mu[U^*](z) - \eta_\delta \, \mathbb{E}_z\left[(\mu[U^*](z) + W)W\right] \\
&\stackrel{(i)}{=} \mu[U^*](z) - \eta_\delta \, \sigma[U^*](z)^2\\
&\le \mu[U^*](z) - \sqrt{2\sigma[U^*](z)^2 \delta} + \delta |U|_{\mathrm{span}},
\end{aligned}
\end{equation}
where $(i)$ uses $\E_z W = 0$ and \eqref{eqn:tu_EW2_var}, and the last inequality follows from the lower bound on $\eta_\delta$. 

Finally, combining \eqref{eq:feasible}, \eqref{eq:Q_etabound}, and the feasibility of $Q_{\eta_\delta}$, we obtain
\begin{equation}\label{eq:upper_bound}
\inf_{P \in \mathcal{P}_z} \mathbb{E}_P[v[U^*(X)]]
\le \mu[U^*](z) - \sqrt{2 \delta}\sigma[U^*](z) + \delta |U|_{\mathrm{span}}.
\end{equation}

\paragraph{Lower bound.}
Observe that
\begin{equation}\label{eq:inf_deviation}
\begin{aligned}
\inf_{Q \in \mathcal{P}_{z}} \mathbb{E}_Q[v[U^{*}](X)] 
&= \E_z[v[U^{*}](X)] + \inf_{Q \in \mathcal{P}_{z}}\crbk{\mathbb{E}_Q[v[U^{*}](X)]- \E_z[v[U^{*}](X)]}\\
&=\mu[U^*](z)+\inf_{Q \in \mathcal{P}_{z}} \mathbb{E}_Q[W]\\
&=\mu[U^*](z)-\sup_{Q \in \mathcal{P}_{z}} \mathbb{E}_Q[-W]
\end{aligned}
\end{equation}
Therefore, a lower bound on $\inf_{Q \in \mathcal{P}_z} \mathbb{E}_Q[v[U^*(X)]]$ can be achieved by controlling the worst-case deviation term $\sup_{Q \in \mathcal{P}_{z}} \mathbb{E}_Q[-W]$. 

To this end, we apply the Donsker--Varadhan variational formula (see, e.g.~\citet{Nguyen2010}), yielding that for any probability measure $Q$ on $S$, 
$$D_{\mathrm{KL}}(Q \,\|\, P)\geq \E_Q[f(X)] - \log \E_P\sqbk{e^{f(X)}}$$
for any bounded measureable $f:S\ra\R$. For any $Q \in \mathcal{P}_z$ and $\lambda > 0$, applying the formula to $f(x) = \lambda( \E_z[ v[U^*](X)]- v[U^*](x)) = -\lambda W$, we have
$$\lambda \, \mathbb{E}_Q[-W]
\le 
D_{\mathrm{KL}}(Q \,\|\, P_0(\cdot | z)) 
+ \log \mathbb{E}_z[e^{-\lambda W}]. $$
Since $Q \in \mathcal{P}_z$, $D_{\mathrm{KL}}(Q \,\|\, P_0(\cdot | z)) \le \delta$. Then, dividing by $\lambda > 0$, we obtain
\begin{equation}\label{eq:DonskerFormula}
\mathbb{E}_Q[-W]
\le 
\frac{\delta}{\lambda}
+ \frac{1}{\lambda} \log \mathbb{E}_z[e^{-\lambda W}].
\end{equation}

We now bound the log-moment generating function. Note that $\E_zW = 0$ and $|W|\leq \spnorm{U^*}$. Then, using the Bernstein inequality \citep[Proposition 2.10]{Wainwright_2019} we have
\[
\log \mathbb{E}_z[e^{-\lambda W}]
\le 
\frac{\lambda^2 \sigma[U^*](z)^2}{2\crbk{1 - \lambda |U^*|_{\mathrm{span}}}}.
\]
for any $\lambda \in (0, 1/|U^*|_{\mathrm{span}})$. Plugging back into \eqref{eq:DonskerFormula} yields
\begin{equation}\label{eqn:tu_EQ_W_bd}
\mathbb{E}_Q[-W]
\le 
\frac{\delta}{\lambda}
+ 
\frac{\lambda \sigma[U^*](z)^2}{2\crbk{1 - \lambda |U^*|_{\mathrm{span}}}},
\end{equation}
for all $Q\in\cP_z$. 

To get an upper bound for \eqref{eq:pointwise_goal}, we choose
$$
\lambda
:=
\frac{\sqrt{2\delta}}{\sigma[U^*](z) + |U^*|_{\mathrm{span}}\sqrt{2\delta}},
$$
which satisfies $ 0< \lambda < 1/|U^*|_{\mathrm{span}}$. Plug this $\lambda$ into \eqref{eqn:tu_EQ_W_bd}, we get
\[
\mathbb{E}_Q[-W]
\le 
\sqrt{2 \delta}\sigma[U^*](z)
+ 
\delta \,|U^*|_{\mathrm{span}},
\]
for all $Q \in \mathcal{P}_z$. Taking the supremum gives
\begin{equation}\label{eq:sup_bound}
\sup_{Q \in \mathcal{P}_z} \mathbb{E}_Q[-W]
\le 
\sqrt{2 \delta}\sigma[U^*](z)
+ 
\delta \,|U^*|_{\mathrm{span}}.
\end{equation}
Then, combining \eqref{eq:sup_bound} with \eqref{eq:inf_deviation}, we obtain
\begin{equation}\label{eq:lower_bound}
\inf_{Q \in \mathcal{P}_z} \mathbb{E}_Q[v[U^*(X)]]
\ge 
\mu[U^*](z)
- \sqrt{2 \delta}\sigma[U^*](z)
- \delta\,|U^*|_{\mathrm{span}}.
\end{equation}

\paragraph{Concluding Theorem \ref{thm:approx-error}. }Finally, we combine the upper bound \eqref{eq:upper_bound} and lower bound \eqref{eq:lower_bound} to obtain 
$$\inf_{P \in \mathcal{P}_z}\mathbb{E}_P[v[U^*(X)]] \in \sqbk{\mu[U^*](z)
- \sqrt{2\delta}\sigma[U^*](z)  \pm \delta\spnorm{U^*}}.$$
This implies \eqref{eq:pointwise_goal} and hence completes the proof.
\end{proof}

\section{Proof of Theorem~\ref{thm:convergence-clt}}

We first establish the error bound with the $\varepsilon$-stabilized approximate Bellman operator. Note that by Theorem \ref{thm:contraction_eps}, both $\cR$ and $\cR_\epsilon$ are $L$ contraction. So, by \eqref{eq:reduction}, we have that
$$\begin{aligned}\norminf{U^*-U^*_\varepsilon} &\leq\frac{1}{1-L}\norminf{\cR[U^*] - \cR_\epsilon[U^*]} \\
&= \frac{\gamma\sqrt{2\delta}\norminf{ \sigma[U^*]- \sigma_\varepsilon[U^*]}}{1-L}\\
&\leq \frac{\gamma\sqrt{2\delta\epsilon}}{1-L} , 
\end{aligned}$$
where the last step we used that for $a\ge0$, $(\sqrt{a+\epsilon}-\sqrt a)^2 = 2a+\epsilon-2\sqrt{a(a+\epsilon)}\leq \epsilon.$

We then split the rest of the proof into two parts. First, we show the almost sure convergence of the iterates $(U_n,m_n,g_n)\ra (U^*_\varepsilon,m^*_\varepsilon,g^*_\varepsilon)$. Then, we establish the central limit theorem that complements this convergence result.

Before we proceed, we first note that for the algorithm to be valid, we need to guarantee that $\sigma_{n,\varepsilon}$ is well-defined. In particular, we inductively check that 
$$g_n(z) - m_n(z)^2 \geq 0.$$

From the specification of Algorithm \ref{alg:mvsa}, we have that $g_1 = 1$ and $m_1 = 0$, the vector of all ones and zeros. So, for the base case, $g_1(z) - m_1(z)^2 = 1\geq 0$. 

To show the induction step, we note that from \eqref{eqn:def_fast_update}, the fast updates can be written as
\begin{equation}\label{eqn:fast_update_rewrite}
\begin{aligned}
  m_{n+1}(z) &= (1-\beta_n)\,m_n(z) + \beta_n\,Z_{n+1}(z),\\
  g_{n+1}(z) &= (1-\beta_n)\,g_n(z) + \beta_n\,Z_{n+1}(z)^2.
\end{aligned}
\end{equation}
A direct expansion yields
\begin{align*}
  g_{n+1}(z) - m_{n+1}(z)^2
  &= (1-\beta_n)\,g_n(z) + \beta_n\,Z_{n+1}(z)^2 - \crbk{(1-\beta_n)\,m_n(z) + \beta_n\,Z_{n+1}(z)}^2 \\
  &= (1-\beta_n)\crbk{g_n(z) - m_n(z)^2} + \beta_n(1-\beta_n)\crbk{Z_{n+1}(z) - m_n(z)}^2 \\
  &\geq 0
\end{align*}
which implies $g_n(z) \geq m_n(z)^2$ for all $n$.

\subsection{Proof of Almost Sure Convergence}\label{app:convergence}
In this section, we prove part (i) of Theorem~\ref{thm:convergence-clt}. 

The argument proceeds by first establishing boundedness of the iterates and controlling the residual errors of the fast variables. We then apply classical convergence results for supermartingale-like processes to show convergence of the fast iterates to a stochastic target. The convergence of the fast and slow iterates to the fixed point is subsequently obtained through a continuity argument.

\paragraph{Boundedness of $(U_n,g_n,m_n)$. }
We establish uniform boundedness of the iterates. Specifically, we define
\begin{equation}\label{eq:B-def}
  B := \frac{r_{\max} + \gamma\sqrt{2\delta\varepsilon}}{1 - \gamma(1 + \sqrt{2\delta})} = \frac{r_{\max} + \gamma\sqrt{2\delta\varepsilon}}{1 - L}.
\end{equation}
We show inductively that for all $n\geq 1$,
\begin{equation}\label{eq:uniform-bounds}
  \|U_n\|_\infty \leq B, \qquad \|m_n\|_\infty \leq B, \qquad 0 \leq \|g_n\|_\infty \leq B^2. 
\end{equation}

From the the initial condition in Algorithm \ref{alg:mvsa}, $B \geq \|U_1\|_\infty \vee \|m_1\|_\infty \vee \sqrt{\|g_1\|_\infty}$. For the induction step, we note that $Z_{n+1}(z) = \max_{b \in \mathcal{A}} U_n(X'_{n+1}(z), b)$. So, assuming that \eqref{eq:uniform-bounds} holds at iteration $n$, it follows that
\[
  |Z_{n+1}(z)| \leq \|U_n\|_\infty \leq B
\]
for all $z\in\cZ$. From \eqref{eqn:fast_update_rewrite} and $\beta_n\leq1$, we have
\[
  \|m_{n+1}\|_\infty \leq (1-\beta_n)\|m_n\|_\infty + \beta_n\|Z_{n+1}\|_\infty.
\]
and 
\[
  \|g_{n+1}\|_\infty \leq (1-\beta_n)\|g_n\|_\infty + \beta_n\|Z_{n+1}\|_\infty^2.
\]
So, $\|m_{n+1}\|_\infty \leq B$ and $0 \leq \|g_n\|_\infty \leq B^2$. This shows the last two bounds in \eqref{eq:uniform-bounds}. 

Next, rewrite the slow update in \eqref{eqn:slow_update} as
\begin{equation}\label{eqn:slow_update_rewrite}
  U_{n+1} = (1-\alpha_n)\,U_n + \alpha_n\crbk{r + \gamma\,m_n - \gamma\sqrt{2\delta}\sqrt{g_n - m_n^2 + \varepsilon}}.
\end{equation}
Since $g_n - m_n^2\geq 0$, we have $$\sqrt{g_n - m_n^2 + \varepsilon} \leq \sqrt{\|g_n\|_\infty} + \sqrt{\varepsilon} \leq B + \sqrt{\varepsilon}.$$ Therefore, \eqref{eqn:slow_update_rewrite} and $\alpha_n\leq1$ imply that
\begin{align}\label{eq:U_inequality}
  \|U_{n+1}\|_\infty
  &\leq (1-\alpha_n)\|U_n\|_\infty + \alpha_n\crbk{r_{\max} + \gamma\|m_n\|_\infty + \gamma\sqrt{2\delta}(\|g_n\|_\infty^{1/2} + \sqrt{\varepsilon})} \notag\\
  &\leq (1-\alpha_n)B + \alpha_n\crbk{r_{\max} + \gamma B + \gamma\sqrt{2\delta}(B + \sqrt{\varepsilon})}.
\end{align}
Recalling the definition of $B$ in \eqref{eq:B-def}, we have $(1-L)B = r_{\max} + \gamma \sqrt{2\delta \epsilon}$. So 
$$r_{\max} + \gamma B + \gamma\sqrt{2\delta}(B+\sqrt{\varepsilon}) = r_{\max} + LB + \gamma \sqrt{2\delta \epsilon}= B.$$
Hence, the right-hand side of \eqref{eq:U_inequality} becomes $(1-\alpha_n)B + \alpha_n B = B$. Hence $\|U_{n+1}\|_\infty \leq B$,
completing the induction. All bounds in \eqref{eq:uniform-bounds} are checked.

\paragraph{Fast iterates' convergence to a stochastic target.}
Define the tracking errors
\begin{align*}
  E_n &:= m_n - \mu[U_n], \\
  \widetilde{E}_n &:= g_n - \nu[U_n].
\end{align*}
We show that $E_n\ra 0$ and $\widetilde{E}_n\ra 0$ a.s. as $n\ra\infty$. 

To show this, we define the martingale noise terms
\[
  \xi_{n+1} := Z_{n+1} - \mu[U_n], \qquad 
  \widetilde{\xi}_{n+1} := Z_{n+1}^2 - \nu[U_n].
\]
and the bias terms
\begin{equation}\label{eqn:bias_term}
  \Delta_n := \mu[U_n] - \mu[U_{n+1}], \qquad \widetilde{\Delta}_n := \nu[U_n] - \nu[U_{n+1}].
\end{equation}
Then the errors can be recursively written as
\begin{equation}\label{eqn:error_recursion}
\begin{aligned}
  E_{n+1} &= m_{n+1} - \mu[U_{n+1}] = (1-\beta_n)\,E_n + \beta_n\,\xi_{n+1} + \Delta_n, \\
  \widetilde{E}_{n+1} &= g_{n+1} - \nu[U_{n+1}] = (1-\beta_n)\,\widetilde{E}_n + \beta_n\,\widetilde{\xi}_{n+1} + \widetilde{\Delta}_n.
\end{aligned}
\end{equation}

To establish the convergence of this recursion, we will analyze the squared error $ E_{n+1}^2$ and $\widetilde{E}_{n+1}^2$ as supermartingale-type processes. Specifically, we introduce the following variant of the Robbins--Siegmund Theorem for the convergence of almost supermartingales. We provide a proof in Appendix \ref{app:proof_of_RS}.

\begin{lemma}[Corollary of \citet{robbins1971convergence}]\label{lem:robbins-siegmund}
Let $\set{V_n:n \geq 1}$ be a nonnegative $\mathcal{F}_n$-adapted process. Suppose there exist nonnegative deterministic sequences $\set{a_n,b_n:n \geq 1}$ such that
\begin{equation}\label{eq:RS-condition}
  \mathbb{E}[V_{n+1} \mid \mathcal{F}_n] \leq (1 - a_n)\,V_n + b_n \qquad \text{a.s.\ for all } n\ge 1.
\end{equation}
Assume further that
\begin{equation}\label{eq:RS-summability}
  \sum_{n=1}^\infty a_n = \infty, \quad \text{and} \quad \sum_{n=1}^\infty b_n < \infty.
\end{equation}
Then $V_n \to 0$ a.s. as $n\ra\infty$.
\end{lemma}

Define the following filtration corresponding to the synchronous sampling process: 
$$\mathcal{F}_n:= \sigma\crbk{U_0, m_0, g_0,\ \{X_k'(s,a) : 1 \le k \le n,\ (s,a) \in \mathcal{S} \times \mathcal{A}\}}.$$ 
To apply Lemma \ref{lem:robbins-siegmund}, we first establish some preliminary bounds. Conditioning on $\cF_n$ and using $X'_{n+1}(z) \sim P_0(\cdot | z)$, we have
\[
  \mathbb{E}[Z_{n+1}(z) \mid \mathcal{F}_n] = \mu[U_n](z),
\quad\text{and}\quad
\mathbb{E}[Z_{n+1}(z)^2 \mid \mathcal{F}_n] = \nu[U_n](z).
\]
Hence,
\begin{equation}\label{eqn:mg_noise}\mathbb{E}[\xi_{n+1} \mid \mathcal{F}_n] = 0,\quad \text{and}\quad \mathbb{E}[\widetilde{\xi}_{n+1} \mid \mathcal{F}_n] = 0;
\end{equation}
i.e. $\xi_n$ and $\widetilde{\xi}_{n}$ are indeed martingale noise sequences. 
Moreover, by \eqref{eq:uniform-bounds}, we have $|Z_{n+1}(z)| \leq B$ a.s., which implies $|\mu[U_n](z)| \leq B$ and $0 \leq |\nu[U_n](z)| \leq B^2$. Consequently,
\begin{equation}\label{eq:noise-bounds}
  \|\xi_{n+1}\|_\infty \leq 2B, \quad\text{and}\quad \|\widetilde{\xi}_{n+1}\|_\infty \leq 2B^2.
\end{equation}

To bound bias terms in \eqref{eqn:bias_term}, we first note that since the slow update \eqref{eqn:slow_update_rewrite} is completely deterministic, $U_{n+1}$ is $\cF_{n}$ measurable, and so are $\Delta_n$ and $\widetilde\Delta_{n}$. 

Next, for any $U, \widetilde{U}$ with $\|U\|_\infty, \|\widetilde{U}\|_\infty \leq B$, we have $|v[U](x)|, |v[\widetilde{U}](x)| \leq B$. Then, for all $z\in \cZ$ and $x\in \cS$, 
\[
  |\mu[U](z) - \mu[\widetilde{U}](z)| \leq \E_z[|v[U](X') - v[\widetilde{U}](X')|] \leq \|U - \widetilde{U}\|_\infty.
\]
and 
\[
  |v[U](x)^2 - v[\widetilde{U}](x)^2| = |v[U](x) - v[\widetilde{U}](x)|\,|v[U](x) + v[\widetilde{U}](x)| \leq 2B\|U - \widetilde{U}\|_\infty.
\]
Therefore, 
\begin{equation}\label{eq:mu-lip}
  \|\mu[U] - \mu[\widetilde{U}]\|_\infty \leq \|U - \widetilde{U}\|_\infty,
\end{equation}
and
\begin{equation}\label{eq:nu-lip}
  \|\nu[U] - \nu[\widetilde{U}]\|_\infty \leq 2B\|U - \widetilde{U}\|_\infty. 
\end{equation}
Moreover, From the slow update \eqref{eqn:slow_update},
\[
U_{n+1} - U_n = \alpha_n\crbk{r + \gamma m_n - \gamma\sqrt{2\delta}\,\sigma_{n,\varepsilon} - U_n},
\]
and using the uniform bounds~\eqref{eq:uniform-bounds}, we obtain the coarse estimate
\begin{equation}\label{eq:incrementU}
  \|U_{n+1} - U_n\|_\infty \leq 2B\,\alpha_n.
\end{equation}
Combining with~\eqref{eq:mu-lip} -\eqref{eq:incrementU} yields
\begin{equation}\label{eq:mismatch-bounds}
  \|\Delta_n\|_\infty = \|\mu[U_n] - \mu[U_{n+1}]\|_\infty \leq 2B\,\alpha_n, \qquad
  \|\widetilde{\Delta}_n\|_\infty = \|\nu[U_n] - \nu[U_{n+1}]\|_\infty \leq 4B^2\,\alpha_n.
\end{equation}

With these bound, we apply Lemma \ref{lem:robbins-siegmund} to $E_{n+1}(z)^2$ and $\widetilde E_{n+1}(z)^2$, respectively. We expand $E_{n+1}(z)^2$ and conditioning on $\mathcal{F}_n$, yielding
\begin{align*}
\mathbb{E}\left[E_{n+1}(z)^2\mid\mathcal F_n\right]
&=\mathbb E\left[\crbk{(1-\beta_n)E_n(z)+\beta_n\xi_{n+1}(z)+\Delta_n(z)}^2 \mid \mathcal F_n\right]\\
&=\mathbb{E}\Big[(1-\beta_n)^2 E_n(z)^2
+ \beta_n^2 \xi_{n+1}(z)^2
+ \Delta_n(z)^2\\
&\hspace{2.6cm}
+ 2(1-\beta_n)E_n(z)\,\beta_n\xi_{n+1}(z)
+ 2(1-\beta_n)E_n(z)\,\Delta_n(z)\\
&\hspace{2.6cm}
+ 2(\beta_n\xi_{n+1}(z))\,\Delta_n(z)\ \Bigm|\ \mathcal F_n\Big]\\
&\stackrel{(i)}{=}
(1-\beta_n)^2 E_n(z)^2 + 2(1-\beta_n)E_n(z)\Delta_n(z)+\Delta_n(z)^2
+ \beta_n^2\,\mathbb E[\xi_{n+1}(z)^2\mid \mathcal F_n]\\
&\stackrel{(ii)}{\le}
(1-\beta_n)E_n(z)^2+\frac{\beta_n}{2}E_n(z)^2+\left(1+\frac{2}{\beta_n}\right)\Delta_n(z)^2
+ \beta_n^2\,\mathbb E[\xi_{n+1}(z)^2\mid \mathcal F_n]\\
&= \left(1-\frac{\beta_n}{2}\right)E_n(z)^2+\left(1+\frac{2}{\beta_n}\right)\Delta_n(z)^2
+ \beta_n^2\,\mathbb E[\xi_{n+1}(z)^2\mid \mathcal F_n]\\
&\stackrel{(iii)}{\le} \left(1-\frac{\beta_n}{2}\right)E_n(z)^2
+\left(1+\frac{2}{\beta_n}\right)4B^2\alpha_n^2
+ 4B^2\beta_n^2. 
\end{align*}
where
\begin{itemize}
    \item $(i)$ uses \eqref{eqn:mg_noise} and hence the cross term involving $\xi_{n+1}(z)$ vanishes.
    \item $(ii)$ comes from two inequalities: with $\beta_n\in[0,1]$, $(1-\beta_n)^2\le 1-\beta_n$ and hence
\[
2(1-\beta_n)|E_n(z)|\,|\Delta_n(z)|
\le 2|E_n(z)|\,|\Delta_n(z)|
\le \frac{\beta_n}{2}E_n(z)^2+\frac{2}{\beta_n}\Delta_n(z)^2,
\]
which applies the fact that $2ab\le \frac{\beta_n}{2}a^2+\frac{2}{\beta_n}b^2$, with $a=|E_n(z)|$, $b=|\Delta_n(z)|$.
    \item $(iii)$ uses $|\Delta_n(z)|\le 2B\alpha_n$ and $\mathbb E[\xi_{n+1}(z)^2\mid \mathcal F_n]\le 4B^2$.
\end{itemize}
Since $\beta_n \leq 1$ implies $1 + 2/\beta_n \leq 3/\beta_n$, there exist constants $c_1 = 12B^2$ and $c_2 = 4B^2$ such that
\begin{equation}\label{eq:E-sq-bound}
\mathbb{E}[E_{n+1}(z)^2 \mid \mathcal{F}_n]
\le \crbk{1 - \tfrac{\beta_n}{2}} E_n(z)^2 + c_1 \frac{\alpha_n^2}{\beta_n} + c_2 \beta_n^2,
\end{equation}
The identical argument applied to $\widetilde{E}_{n+1}(z) = (1-\beta_n)\,\widetilde{E}_n(z) + \beta_n\,\widetilde{\xi}_{n+1}(z) + \widetilde{\Delta}_n(z)$ yields 
\begin{equation}\label{eq:Etilde-sq-bound}
  \mathbb{E}\sqbk{\widetilde{E}_{n+1}(z)^2 \mid \mathcal{F}_n}
  \leq \crbk{1 - \frac{\beta_n}{2}}\,\widetilde{E}_n(z)^2 + \widetilde{c}_1\,\frac{\alpha_n^2}{\beta_n} + \widetilde{c}_2\,\beta_n^2.
\end{equation}
where $\widetilde{c}_1 = 48B^4$ and $\widetilde{c}_2 = 4B^4$.

Then, we check summability before applying Lemma \ref{lem:robbins-siegmund}. With the step sizes $\alpha_n = a/(n+a)$ and $\beta_n =  b/(n+a)^{\tau}$ with $\tau \in (1/2, 1)$, we have the following
\begin{enumerate}
\item[(i)] $\sum_{n=1}^\infty \beta_n = \infty$: Since $$\sum_{n=1}^\infty \beta_n = b \sum_{n=a+1}^\infty n^{-\tau} = \infty.$$
\item[(ii)] $\sum_{n=1}^\infty \beta_n^2 < \infty$: We have $\beta_n^2 = b^2(n+a)^{-2\tau}$, and since $2\tau > 1$,
$$\sum_{n=1}^\infty \beta_n^2 = b^2 \sum_{n=a+1}^\infty n^{-2\tau} < \infty.$$
\item[(iii)] $\sum_{n=1}^\infty \alpha_n^2/\beta_n < \infty$: Note that for $n \geq 1$,
$\alpha_n^2/\beta_n = a^2b^{-1} (n+a)^{-(2-\tau)}.$
Since $\tau < 1$, we have $2 - \tau > 1$. Thus, $\sum_{n=1}^\infty \alpha_n^2/\beta_n < \infty$.
\end{enumerate}

Fix $z\in\cZ$. Set $V_n := E_n(z)^2$, $a_n := \beta_n/2 \in [0,1]$, and $b_n := c_1\,\alpha_n^2/\beta_n + c_2\,\beta_n^2 \geq 0$. Then~\eqref{eq:E-sq-bound} reads
\[
  \mathbb{E}[V_{n+1} \mid \mathcal{F}_n] \leq (1 - a_n)\,V_n + b_n,
\]
with $\sum_n a_n = \infty$ by (i) and $\sum_n b_n < \infty$ by (ii) and (iii). By Lemma~\ref{lem:robbins-siegmund}, $V_n \to 0$ a.s., which is $E_n(z) \to 0$ a.s. Applying the same argument to ~\eqref{eq:Etilde-sq-bound} yields $\widetilde{E}_n(z) \to 0$ a.s. 

Since there are finitely many coordinates $z\in\cZ$, we conclude that
\begin{equation}\label{eq:tracking-convergence}
  \|E_n\|_\infty = \max_{z\in\cZ} |E_n(z)| \to 0  \quad\text{and}\quad
  \|\widetilde{E}_n\|_\infty = \max_{z\in\cZ} |\widetilde{E}_n(z)| \to 0 \quad\text{a.s.}
\end{equation}

\paragraph{Almost sure convergence to the fixed point.} Define the perturbation
\[
  \eta_n := \gamma E_n - \gamma\sqrt{2\delta}\crbk{\sqrt{g_n - m_n^2 + \varepsilon} - \sqrt{\nu[U_n] - \mu[U_n]^2 + \varepsilon}}.
\]
All square roots and squares above are interpreted elementwise in $z$. By direct algebra, the slow time-scale update rule \eqref{eqn:slow_update} can be written as
\begin{equation}\label{eq:rewrite_U_iter_with_eta}
U_{n+1} = U_n + \alpha_n \crbk{\cR_\varepsilon[U_n] - U_n + \eta_n}.
\end{equation}
To proceed, we analyze the convergence of this recursion by first showing the convergence of the error term $\eta_n\ra 0$ a.s. Then, using the contraction property of $\cR_\varepsilon$, we extend the convergence to $U_n\ra U^*_\epsilon$ a.s. 

To show the convergence of $\eta_n$, we define $x_n := g_n - m_n^2+\varepsilon$ and $y_n := \nu[U_n] - \mu[U_n]^2+\varepsilon$. Consider
\begin{align*}
  x_n - y_n
  &= (g_n - \nu[U_n]+\varepsilon) - (m_n^2 - \mu[U_n]^2+\varepsilon) \\
  &= \widetilde{E}_n - (m_n - \mu[U_n])(m_n + \mu[U_n]) \\
  &= \widetilde{E}_n - E_n\,(m_n + \mu[U_n]).
\end{align*}
Hence, by the uniform bounds~\eqref{eq:uniform-bounds},
\[
  \|x_n - y_n\|_\infty \leq \|\widetilde{E}_n\|_\infty + 2B\|E_n\|_\infty.
\]
We note that for $x, y \geq 0$, $|x-y| + |y| \geq |x|$, which implies $\sqrt{|x-y|} + \sqrt{|y|} \geq \sqrt{|x-y| + y} \geq \sqrt{x}$. Rearranging gives the bound
\[
  |\sqrt{x} - \sqrt{y}| \leq \sqrt{|x-y|}.
\]
Therefore, we have 
\[
  \bigl\|\sqrt{x_n} - \sqrt{y_n}\bigr\|_\infty \le \sqrt{\bigl\|x_n - y_n\bigr\|_\infty}
\]
where, by\eqref{eq:tracking-convergence}, the r.h.s. goes to 0 a.s. as $n\ra\infty$. 
Therefore, 
\begin{equation}\label{eqn:eta_converge}\norminf{\eta_n} \leq  \gamma\norminf{ E_n} +\gamma\sqrt{2\delta}\norminf{\sqrt{x_n}-\sqrt{y_n}}  \ra 0 
\end{equation}
a.s. as $n\ra\infty$. 

Next, to streamline the proof, we introduce the following lemma that guarantees the convergence of $U_n\ra U^*_\varepsilon$ by the convergence of $\eta_n$.  The proof of Lemma \ref{lem:slow_convergence} is deferred to Appendix \ref{app:proof_of_slow_conv}. 
\begin{lemma}
\label{lem:slow_convergence}
Let $\{U_n:n\ge 1\}$ satisfies the recursion \eqref{eq:rewrite_U_iter_with_eta}
with $\sum_{n} \alpha_n = \infty$ and $\alpha_n \to 0$, and 
$\|\eta_n\|_\infty \to 0$ a.s. Then $\|U_n - U^*_\varepsilon\|_\infty \to 0$
a.s.
\end{lemma}

Equipped with this lemma, we show the convergence of the fast iterates. Write
\begin{align*}
  m_n - \mu[U_{\varepsilon}^{*}] &= \crbk{m_n - \mu[U_n]} + \crbk{\mu[U_n] - \mu[U_{\varepsilon}^{*}]}, \\
  g_n - \nu[U_{\varepsilon}^{*}] &= \crbk{g_n - \nu[U_n]} + \crbk{\nu[U_n] - \nu[U_{\varepsilon}^{*}]}.
\end{align*}
The first terms converge to $0$ a.s. by~\eqref{eq:tracking-convergence}. The second terms converge to $0$ a.s. by the Lipschitz continuity~\eqref{eq:mu-lip} and~\eqref{eq:nu-lip} together with $U_n \to U_{\varepsilon}^{*}$ in Lemma \ref{lem:slow_convergence}. Hence
\[
  m_n \to \mu[U_{\varepsilon}^{*}] = m_{\varepsilon}^{*} \quad\text{and}\quad 
  g_n \to \nu[U_{\varepsilon}^{*}] = g_{\varepsilon}^{*} \quad\text{a.s.}
\]
This completes the proof of part (i) of Theorem~\ref{thm:convergence-clt}. \qed

\subsection{Proof of the Asymptotic Normality}

To establish the asymptotic normality of the proposed algorithm, we would like to apply the following two-time-scale SA central limit theorem \citet[Theorem~1]{MokkademPelletier2006}. Specifically, the theorem concerns the two-time-scale SA updates of the form 
\begin{equation}\label{eqn:SA_update_MP}
\begin{aligned}
\theta_{n+1} &= \theta_n + \alpha_n X_{n+1}, \\
\mu_{n+1} &= \mu_n + \beta_n Y_{n+1},
\end{aligned}
\end{equation}
where $\theta_n \in \mathbb{R}^d$ and $\mu_n \in \mathbb{R}^{d'}$. The $\set{X_n,Y_n:n\geq1}$ are designed to find the unique zero $0 = f(\theta,\mu)$ and $0 = h(\theta,\mu)$. We first state the required assumptions. 

\begin{assumption}\label{assump:SA_MP}
    Assume the following conditions hold for the updates \eqref{eqn:SA_update_MP}. 

\begin{enumerate}
\item[(A1)] $\theta_n \to \theta^*$ and $\mu_n \to \mu^*$ almost surely.

\item[(A2)]
\begin{enumerate}
\item[(i)] There exists a neighborhood $\mathcal U$ of $(\theta^*,\mu^*)$ such that, for all $(\theta,\mu)\in\mathcal U$,
\begin{align}
\begin{bmatrix}
f(\theta,\mu)\\
h(\theta,\mu)
\end{bmatrix}
=
\begin{bmatrix}
Q_{11} & Q_{12}\\
Q_{21} & Q_{22}
\end{bmatrix}
\begin{bmatrix}
\theta-\theta^*\\
\mu-\mu^*
\end{bmatrix}
+
O\left(
\left\|
\begin{bmatrix}
\theta-\theta^*\\
\mu-\mu^*
\end{bmatrix}
\right\|^2
\right).\label{eqn:smoothness_Jacobian}
\end{align}
\item[(ii)] Define 
\[
H := Q_{11} - Q_{12} Q_{22}^{-1} Q_{21}, 
\quad \text{and} \quad 
\Lambda(A) := - \max \{ \Re(\lambda) : \lambda \in \mathrm{Spectrum}(A) \}.
\]
Then $\Lambda(H) > 0$ and $\Lambda(Q_{22}) > 0$.
\end{enumerate}

\item[(A3)]
\begin{enumerate}
\item[(i)] The step sizes satisfy $\alpha_n=a n^{-c},\beta_n=b n^{-d},$ with $a>0$, $b>0$, and
\(
\frac12 < d < c \le 1.
\)
\item[(ii)] If $c=1$, then
\(
a  > \frac{1}{2\Lambda(H)}.
\)
\end{enumerate}

\item[(A4)] The error-contaminated observations admit the decomposition
\begin{align*}
X_{n+1} &= f(\theta_n,\mu_n) + \psi_n^{(\theta)} + V_{n+1},\\
Y_{n+1} &= h(\theta_n,\mu_n) + \psi_n^{(\mu)} + W_{n+1},
\end{align*}
and, denoting by $\mathcal F_n$ the natural filtration, the following hold:
\begin{enumerate}
\item[(i)] $\mathbb E[V_{n+1}\mid \mathcal F_n]=0$ and $\mathbb E[W_{n+1}\mid \mathcal F_n]=0$ almost surely.
\item[(ii)] There exists a positive semidefinite matrix that is partitioned as
\[
\Gamma=
\begin{bmatrix}
\Gamma_{11} & \Gamma_{12}\\
\Gamma_{21} & \Gamma_{22}
\end{bmatrix}\in \R^{(d+d') \times (d+ d')}
\]
such that
\[
\mathbb E\left[
\begin{bmatrix}
V_{n+1}\\
W_{n+1}
\end{bmatrix}
\begin{bmatrix}
V_{n+1}^\top & W_{n+1}^\top
\end{bmatrix}
\Biggm| \mathcal F_n
\right]
\to \Gamma
\qquad\text{almost surely.}
\]
\item[(iii)] There exists $m>2/a$ such that
\[
\sup_n \mathbb E\left[\|V_{n+1}\|^m \mid \mathcal F_n\right] < \infty
\quad \text{and}\quad
\sup_n \mathbb E\left[\|W_{n+1}\|^m \mid \mathcal F_n\right] < \infty
\]
almost surely.
\item[(iv)] The remainder terms satisfy
\begin{align*}
\psi_n^{(\theta)} &= r_n^{(\theta)} + O\left(\|\theta_n-\theta^*\|^2+\|\mu_n-\mu^*\|^2\right),\\
\psi_n^{(\mu)} &= r_n^{(\mu)} + O\left(\|\theta_n-\theta^*\|^2+\|\mu_n-\mu^*\|^2\right),
\end{align*}
with $\|r_n^{(\theta)}\|+\|r_n^{(\mu)}\| = o(\sqrt{\beta_n})$ almost surely.
\end{enumerate}
\end{enumerate}
\end{assumption}
\begin{remark}
    In \citet{MokkademPelletier2006}, the authors state Assumption A4 (ii) in terms of a ``positive'' matrix $\Gamma$. A careful reading of the proof shows that positivity here refers to the positive semidefiniteness of $\Gamma$ as a covariance matrix.
\end{remark}

We now state the main result in \citet{MokkademPelletier2006} for reference.

\begin{theorem}[\citet{MokkademPelletier2006}]\label{thm:MP}

Suppose Assumption \ref{assump:SA_MP} holds for the SA updates \eqref{eqn:SA_update_MP}. Define
\[
\Gamma_\theta
=
\Gamma_{11}
+Q_{12}Q_{22}^{-1}\Gamma_{22}(Q_{22}^{-1})^\top Q_{12}^\top
-\Gamma_{12}(Q_{22}^{-1})^\top Q_{12}^\top
-Q_{12}Q_{22}^{-1}\Gamma_{21},
\]
and let $\Sigma_\theta$ and $\Sigma_\mu$ be the unique solutions of the Lyapunov equations
\[
\left(H+\frac{\mathbf 1_{\{c=1\}}}{2a}I\right)\Sigma_\theta
+
\Sigma_\theta
\left(H^\top+\frac{\mathbf 1_{\{c=1\}}}{2a}I\right)
=
-\Gamma_\theta,
\quad \text{and}\quad 
Q_{22}\Sigma_\mu + \Sigma_\mu Q_{22}^\top = -\Gamma_{22}.
\]
Then
\[
\begin{bmatrix}
\alpha_n^{-1/2}(\theta_n-\theta^*)\\
\beta_n^{-1/2}(\mu_n-\mu^*)
\end{bmatrix}
\Rightarrow
\mathcal N\left(
0,
\begin{bmatrix}
\Sigma_\theta & 0\\
0 & \Sigma_\mu
\end{bmatrix}
\right).
\]
\end{theorem}

Equipped with this Theorem, we present our proof of part (ii) of Theorem \ref{thm:convergence-clt}. 

\begin{proof}[Proof of part (ii) of Theorem \ref{thm:convergence-clt}]
We first setup the the two-time-scale SA recursion induced by Algorithm \ref{alg:mvsa}
in the two-time-scale standard form as presented in Theorem \ref{thm:MP}. Specifically, write

\begin{align}
    U_{n+1} &= U_n + \alpha_n\, f(U_n, (m_n,g_n)), \label{eq:SA_slow}\\
    \begin{bmatrix} m_{n+1} \\ g_{n+1} \end{bmatrix}
    &= \begin{bmatrix} m_{n} \\ g_{n} \end{bmatrix} + \beta_n\crbk{h(U_n, (m_n,g_n)) + W_{n+1}}, \label{eq:SA_fast}
\end{align}
where the drift functions are defined by
\begin{align}
    f(U, (m, g)) &:= r + \gamma\, m - \gamma\sqrt{2\delta}\,\sqrt{g - m^2+\varepsilon} - U,\label{eq:drift_f}\\
    h(U, (m, g)) &:= \begin{bmatrix} \mu[U] - m \\ \nu[U] - g \end{bmatrix} \label{eq:drift_h}.
\end{align}
Moreover, the noise term driving the fast recursion is given by
\begin{equation}\label{eq:noise_W}
    W_{n+1} := \begin{bmatrix} Z_{n+1} - \mu[U_n] \\ Z_{n+1}^2 - \nu[U_n] \end{bmatrix},
\end{equation}
where $Z_{n+1}(z) := v[U_n](X'_{n+1}(z))$ for each $z\in\cZ$. We note that functions in \eqref{eq:drift_f} are applied element-wise. Under this setup, we verify the assumptions A1--A4.

\subsubsection*{A1: Almost sure convergence}
By Theorem~\ref{thm:contraction_eps}, $\cR_\varepsilon$ admits a unique fixed point $U^{*}_\varepsilon$, and we define
\(
m^{*}_\varepsilon := \mu[U^*_\varepsilon],\) \(g^{*}_\varepsilon := \nu[U^*_\varepsilon]
\), so that 
\begin{align*}
    f(U^{*}_\varepsilon, (m^{*}_\varepsilon, g^{*}_\varepsilon)) = r + \gamma\, m^{*}_\varepsilon - \gamma\sqrt{2\delta}\sqrt{g^{*}_\varepsilon - m^{*2}_\varepsilon + \varepsilon} - U^{*}_\varepsilon = 0,
\end{align*}
as well as
\begin{align*}
    h(U^{*}_\varepsilon, (m^{*}_\varepsilon, g^{*}_\varepsilon)) = \begin{bmatrix} \mu[U^*_\varepsilon] - m^{*}_\varepsilon \\ \nu[U^*_\varepsilon] - g^{*}_\varepsilon \end{bmatrix} = 0.
\end{align*}
Hence $(U^{*}_\varepsilon, (m^{*}_\varepsilon, g^{*}_\varepsilon))$ is a fixed point of the drift functions $f$ and $h$. Moreover, by the proven part (i) of Theorem \ref{thm:convergence-clt}, 
\begin{equation}\label{eq:as_conv}
    (U_n, m_n, g_n) \xrightarrow{\text{a.s.}} (U^{*}_\varepsilon, m^{*}_\varepsilon, g^{*}_\varepsilon).
\end{equation}
This verifies A1. 

\subsubsection*{A2: Local smoothness and Hurwitz linearization} We note that in this step, the additional Assumption \ref{ass:unique-greedy} is applied to guarantee local differentiability at the fixed point, an integral requirement for a Gaussian CLT. As in Remark \ref{rmk:unique_greedy_clt}, one cannot expect a Gaussian CLT if there are multiple optimal actions under $U^*_\varepsilon$ at the same state. 
\paragraph{Local smoothness.} We first verify the local smoothness of the drift functions $f$ and $h$ in a neighborhood of the fixed point $(U^{*}_\varepsilon, (m^{*}_\varepsilon, g^{*}_\varepsilon))$. To obtain the expansion in \eqref{eqn:smoothness_Jacobian}, it suffices to check that $f,h$ are locally $C^2$. The dependence on $U$ enters through two components: the max operator $v[U]$ and the nonlinear term $(m,g) \mapsto \sqrt{g - m^2 + \varepsilon}$. We address each in turn.

With Assumption~\ref{ass:unique-greedy}, we have that at $U^*_\varepsilon$, there is a strictly positive action value gap:
\[
  \Delta_{\min} := \min_{s \in \mathcal{S}} \crbk{ U_{\varepsilon}^{*}(s, a^*(s)) - \max_{b \neq a^*(s)} U_{\varepsilon}^{*}(s,b) } > 0.
\]
Thus, for any $s \in \cS$ and $b \neq a^*(s)$
\begin{align*}
        U(s, a^*(s)) - U(s, b)
        &= \crbk{U^{*}_\varepsilon(s, a^*(s)) - U^{*}_\varepsilon(s, b)}
        + \crbk{U(s, a^*(s)) - U^{*}_\varepsilon(s, a^*(s))}
        - \crbk{U(s, b) - U^{*}_\varepsilon(s, b)} \\
        &\geq \Delta_{\min} - 2\|U - U^{*}_\varepsilon\|_\infty > 0,
\end{align*}
whenever $\|U - U^{*}_\varepsilon\|_\infty < \Delta_{\min}/2$. Therefore $a^*(s)$ remains the unique maximizer, and $v[U](s)$ is a linear function of $U$ in this neighborhood:
\[
v[U](s)
= \sum_{a\in\cA} \pi^*(a\mid s)\, U(s,a),
\quad
\text{where } \quad
\pi^*(a\mid s)
=
\begin{cases}
1, & a = a^*(s),\\
0, & \text{otherwise}.
\end{cases}
\]

Therefore, in the neighborhood $\{U : \|U - U^{*}_\varepsilon\|_\infty < \Delta_{\min}/2\}$, the mapping $U \mapsto v[U]$ is linear. Consequently, the mappings $U \mapsto \mu[U]$ and $U \mapsto \nu[U]$ are $C^\infty$. In particular, all partial derivatives of them exist and are locally bounded.

For the nonlinear term $(m,g) \mapsto \sqrt{g - m^2 + \varepsilon}$, if $\varepsilon > 0$, then $g - m^2 + \varepsilon$ is automatically bounded away from zero, so the square root is well defined and smooth in a neighborhood of $U^{*}_\varepsilon$ without any additional assumptions. If $\varepsilon = 0$, the non-degeneracy condition is in force: $\sigma[U^{*}_\varepsilon](z) > 0$ for all $z \in \cZ$. Define the minimum variance level at the root by
\[
 {\sigma} := \min_{z\in\cZ} \sigma[U^{*}_\varepsilon](z) > 0.
\]

\noindent
Since $\mu$ and $\nu$ are continuous functions of $U$, the mapping
\(
U \mapsto \nu[U](z) - \mu[U](z)^2
\)
is continuous for each $z\in\cZ$. Moreover,
\[
\nu[U^*_\varepsilon](z) - \mu[U^*_\varepsilon](z)^2
= \sigma[U^{*}_\varepsilon](z)^2
\ge  \sigma^2 > 0.
\]
Hence, by continuity, there exists $\eta>0$ such that for all $U$ satisfying
$\|U-U^{*}_\varepsilon\|_\infty<\eta$,
\[
\nu[U](z) - \mu[U](z)^2
\ge  \sigma^2 - \frac{ \sigma^2}{2} = \frac{ \sigma^2}{2},
\qquad \forall z\in\cZ.
\]
Then on the open set $\{(m,g) : g - m^2 \geq  {\sigma}^2/2\}$, the mapping $(m,g) \mapsto \sqrt{g - m^2+\varepsilon}$ is $C^\infty$ with bounded derivatives. 

\paragraph{Compute linearization and verify contraction.}
With the smoothness of $f$ and $h$ established, we now compute the Jacobian matrix at the fixed point $(U^{*}_\varepsilon, (m^{*}_\varepsilon, g^{*}_\varepsilon))$ and show that it is Hurwitz. Recall the definition of $Q$ matrices from the linearization \eqref{eqn:smoothness_Jacobian}. We compute them as follows.  

\begin{itemize}
    \item \textbf{Block \(Q_{11}\).}
Since \(f(U,(m,g))\) depends on \(U\) only through the linear term \(-U\), we have
\[
Q_{11}
:=
\nabla_U f(U^{*}_\varepsilon, (m^{*}_\varepsilon, g^{*}_\varepsilon))
=
-I_d.
\]

\item \textbf{Block \(Q_{22}\).}
Since \(h(U,(m,g))=(\mu[U]-m,\ \nu[U]-g)\), differentiating with respect to \((m,g)\) yields
\[
Q_{22}
:=
\nabla_{(m,g)} h(U^{*}_\varepsilon, (m^{*}_\varepsilon, g^{*}_\varepsilon))
=
-I_{2d}.
\]
All eigenvalues of \(Q_{22}\) are equal to \(-1\) so that
\(
\Lambda(Q_{22}) > 0.
\)
This verifies the additional requirement on $Q_{22}$ in A2 (ii). 

\item\textbf{Block \(Q_{21}\).}
Differentiating \(h(U,(m,g))\) with respect to \(U\) gives
\[
Q_{21}
:=
\nabla_U h(U^{*}_\varepsilon, (m^{*}_\varepsilon, g^{*}_\varepsilon))
=
\begin{bmatrix}
\nabla_U \mu[U^*_\varepsilon]\\[2mm]
\nabla_U \nu[U^*_\varepsilon]
\end{bmatrix}
=
\begin{bmatrix}
C\\
D
\end{bmatrix},
\]
where \(C,D\in\mathbb R^{d\times d}\) have entries, for \(z=(s,a)\) and \(z'=(s',a')\),
\[
C(z,z')
=
\left.\frac{\partial \mu[U](z)}{\partial U(z')}\right|_{U=U^{*}_\varepsilon}
=
P_0(s'\mid z)\,\1\{a'=a(s')\},
\]
and
\[
D(z,z')
=
\left.\frac{\partial \nu[U](z)}{\partial U(z')}\right|_{U=U^{*}_\varepsilon}
=
2\,P_0(s'\mid z)\,U^{*}_\varepsilon\crbk{s',a(s')}\,\1\{a'=a(s')\}.
\]

\item\textbf{Block \(Q_{12}\).}
Differentiating \(f(U,(m,g))\) with respect to \((m,g)\).
Then \(Q_{12}:=\nabla_{(m,g)} f(U^{*}_\varepsilon, (m^{*}_\varepsilon, g^{*}_\varepsilon))\in\mathbb R^{d\times 2d}\)
can be written as
\[
Q_{12}=\begin{bmatrix}A\ B\end{bmatrix},
\]
where \(A,B\in\mathbb R^{d\times d}\) are diagonal matrices with diagonal entries
\[
A(z,z)
=
\gamma
-
\gamma\sqrt{2\delta}\,
\left.\frac{\partial}{\partial m}\sqrt{g-m^2+\varepsilon}\right|_{(m,g)=(m^{*}_\varepsilon(z),g^{*}_\varepsilon(z))}
=
\gamma+\,\frac{\gamma\sqrt{2\delta}m^{*}_\varepsilon(z)}{\sigma_{\varepsilon}^{*}(z)},
\]
and
\[
B(z,z)
=
-\gamma\sqrt{2\delta}\,
\left.\frac{\partial}{\partial g}\sqrt{g-m^2+\varepsilon}\right|_{(m,g)=(m^{*}_\varepsilon(z),g^{*}_\varepsilon(z))}
=
-\,\frac{\gamma\sqrt{2\delta}}{2\,\sigma_{\varepsilon}^{*}(z)},
\]
with $\sigma_{\varepsilon}^{*}(z):=\sqrt{g^{*}_\varepsilon(z)-m^{*}_\varepsilon(z)^2+\varepsilon}$. All off-diagonal entries of \(A\) and \(B\) are zero.

\end{itemize}
As in A2 (ii) we consider
\[
H := Q_{11}-Q_{12}Q_{22}^{-1}Q_{21}=-I_d+Q_{12}Q_{21}.
\]
Writing \(Q_{12}=\begin{bmatrix}A\ B\end{bmatrix}\) and \(Q_{21}=\begin{bmatrix}C\\ D\end{bmatrix}\), it follows that
\[
K:=Q_{12}Q_{21}=AC+BD,
\qquad
H=K-I_d.
\]
Equivalently, in $(z,z')$-entry form we have
\[
H(z,z')
=
- I(z,z')
+ \crbk{A(z,z)\,C(z,z') + B(z,z)\,D(z,z')}.
\]
Expanding $A,B,C,D$ yields the explicit representation
\begin{align*}
H(z,z') = -I(z,z')  + \crbk{1 + \frac{\sqrt{2\delta}m_{\varepsilon}^*(z)}{\sigma_{\varepsilon}^*(z)} -\frac{\sqrt{2\delta}U_{\varepsilon}^*(z')}{\,\sigma_{\varepsilon}^*(z)}}\gamma Q_0(z,z').
\end{align*}
as specified in \eqref{eq:H}.

To complete the verification of A2, it remains to show that $H$ is Hurwitz; i.e., all of its eigenvalues have strictly negative real parts. Rather than checking this directly from the explicit form of $H$---which is possible, but would require careful algebra involving the properties of $Q_0$, $U^*_\varepsilon$, and $m_\varepsilon^*$---we instead exploit the global contraction of the Bellman update to deduce the required spectral property.

Under the smoothness established above, the operator $\mathcal{R_\varepsilon}$ is differentiable in the neighborhood of $U^{*}_\varepsilon$. Specifically, fix any $h \in \mathbb{R}^d$,
\[
\nabla \mathcal{R}_\varepsilon(U) h
= \lim_{t \to 0} \frac{\mathcal{R}_\varepsilon(U + th) - \mathcal{R}_\varepsilon(U)}{t}.
\]
Taking $\|\cdot\|_\infty$ and by Theorem \ref{thm:contraction}, there exists $L \in (0,1)$ such that 
\[
\left\| \frac{\mathcal{R}_\varepsilon(U + th) - \mathcal{R}_\varepsilon(U)}{t} \right\|_\infty
\le \frac{L \|th\|_\infty}{|t|}
= L \|h\|_\infty,
\]
for all $t \neq 0$. Passing to the limit gives
\[
\|\nabla \mathcal{R}_\varepsilon(x) h\|_\infty \le L \|h\|_\infty,
\]
and hence we conclude
\begin{equation}\label{eq:derivative-bound}
\|\nabla \mathcal{R}_\varepsilon(x)\|_{\infty \to \infty}
= \sup_{\|h\|_\infty = 1} \|\nabla \mathcal{R}_\varepsilon(x) h\|_\infty
\le L.
\end{equation}
We also observe that the $U$-Jacobian of $\mathcal{R_\varepsilon}$ at $U^{*}_\varepsilon$ is
\begin{equation}\label{eq:nabla_R}
    \nabla \mathcal{R_\varepsilon}[U^*_\varepsilon] = Q_{12} Q_{21} = K.
\end{equation}
Applying \eqref{eq:derivative-bound} and \eqref{eq:nabla_R} to $\mathcal{R_\varepsilon}$ yields
\[
\|K\|_{\infty \to \infty} = \|\nabla \mathcal{R_\varepsilon}[U^*_\varepsilon]\|_{\infty \to \infty} \leq L.
\]
Let $\lambda$ be any eigenvalue of $K$. Then
\[
|\lambda| \leq \rho(K) \leq \|K\|_{\infty \to \infty} \leq L,
\]
where $\rho(K)$ denotes the spectral radius. Hence every eigenvalue of $H = K - I_d$ is of the form $\lambda - 1$ and satisfies $\Re(\lambda - 1) \leq |\lambda| - 1 \leq L - 1 < 0$. We obtain the quantitative bound
\begin{equation}\label{eq:Lambda_H}
    \Lambda(H) \geq 1 - L > 0.
\end{equation}
Consequently, the drifts admit the local expansion \eqref{eqn:smoothness_Jacobian} with
\[
f(U,m,g)-f(U^*_\varepsilon,m^*_\varepsilon,g^*_\varepsilon)
=
Q_{11}(U-U^*_\varepsilon)
+
Q_{12}
\begin{bmatrix}
m-m^*_\varepsilon\\
g-g^*_\varepsilon
\end{bmatrix}
+
O\left(\|(U,m,g)-(U^*_\varepsilon,m^*_\varepsilon,g^*_\varepsilon)\|^2\right),
\]
\[
h(U,m,g)-h(U^*_\varepsilon,m^*_\varepsilon,g^*_\varepsilon)
=
Q_{21}(U-U^*_\varepsilon)
+
Q_{22}
\begin{bmatrix}
m-m^*_\varepsilon\\
g-g^*_\varepsilon
\end{bmatrix}
+
O\left(\|(U,m,g)-(U^*_\varepsilon,m^*_\varepsilon,g^*_\varepsilon)\|^2\right).
\]
\subsubsection*{A3: Step sizes}

Recall that the step sizes are prescribed by Assumption~\ref{ass:step-sizes}; i.e. 
\begin{equation}\label{eq:stepsizes}
    \alpha_n = \frac{a}{n+a}, \qquad \beta_n = \frac{b}{(n+a)^{\tau}}, \qquad \tau \in (1/2, 1),
\end{equation}
with integer $a > 1/(2(1 - L))$ and $0< b\leq a^\tau$. Note that from \eqref{eq:Lambda_H}, $\Lambda(H)\geq 1-L$. So, $a > 1/(2\Lambda(H))$, satisfying A3 (ii). Thus, A3 holds by applying Theorem~\ref{thm:MP} with the shifted index $k=n+a$. Since a constant shift does not affect the asymptotic distribution, the CLT continues to hold with the same covariance structure.

\subsubsection*{A4: Noise and bias in the SA update}
We observe that, under the two-time-scale SA formulation of the proposed algorithm in \eqref{eq:SA_slow} and \eqref{eq:SA_fast}, the corresponding noise and bias sequences $\set{V_n,\psi_n^{(\theta)},\psi_n^{(\mu)}}$ in A4 of Theorem~\ref{thm:MP} are identically zero. So, we only need to check A4 (i)--(iii) for the $\set{W_n}$ sequence defined in \eqref{eq:noise_W}.

By the sampling mechanism \eqref{eqn:sample_Z} of the algorithm, for each $z\in\cZ$ we have
\[
\mathbb E\left[ Z_{n+1}(z)\mid \mathcal F_n\right]
=
\mu[U_n](z)\quad \text{and}
\quad
\mathbb E\left[ Z_{n+1}(z)^2\mid \mathcal F_n\right]
=
\nu[U_n](z).
\]
Therefore, \(\mathbb E[W_{n+1}\mid \mathcal F_n]=0\); checking (i).

To check (ii), define the conditional covariance matrix of the fast noise by
\begin{equation}\label{eq:Gamma22n}
    \Gamma_{22,n} := \mathbb{E}\sqbk{W_{n+1} W_{n+1}^\top \mid \mathcal{F}_n} \in \mathbb{R}^{2d \times 2d}.
\end{equation}
In the synchronous sampling algorithm, conditional on $\mathcal{F}_n$, the fresh samples $\{X'_{n+1}(z):z \in \cZ\}$ are independent across $z$. Since $\mathbb{E}[W_{n+1} \mid \mathcal{F}_n] = 0$, there are no cross terms between different coordinates. Hence we get
\[
\Gamma_{22,n}
=
\begin{bmatrix}
\Gamma_{mm,n} & \Gamma_{mg,n} \\
\Gamma_{gm,n} & \Gamma_{gg,n}
\end{bmatrix},
\qquad \Gamma_{gm,n}=\Gamma_{mg,n}^\top .
\]
\[
\Gamma_{mm,n}=\diag\crbk{V_n}, \qquad
\Gamma_{mg,n}=\diag\crbk{C_n}, \qquad
\Gamma_{gg,n}=\diag\crbk{W_n}.
\]
where for each $z\in\cZ$ with $Z_n(z):=v[U_n](X'(z))$,
\begin{align*}
V_n(z) &:= \var\crbk{Z_n(z)},\\
C_n(z) &:= \cov\crbk{Z_n(z), Z_n(z)^2},\\
W_n(z) &:= \var\crbk{Z_n(z)^2}.
\end{align*}
Noting that $U_n \to U^{*}_\varepsilon$ almost surely and $Z_n$ is bounded, we take $n \to \infty$ and apply bounded convergence yields 
\begin{align*}
    \Gamma_{22,n}\to \Gamma_{22}
    = \begin{bmatrix}
      \Gamma_{mm} & \Gamma_{mg} \\
      \Gamma_{gm} & \Gamma_{gg}
    \end{bmatrix},
\end{align*}
almost surely. This is exactly the covariance representation stated in \eqref{eq:gamma22} constructed with moment quantities \eqref{eq:VCW}.
Since there is no slow noise term, the full conditional covariance matrix has the form $\Gamma = \text{diag}(0, \Gamma_{22})$, with $\Gamma_{11} = \Gamma_{12} = \Gamma_{21} = 0$.

Then we check the moment condition. By the uniform bound $|Z_{n+1}(z)| \leq B$, we have $|Z_{n+1}(z) - \mu[U_n](z)| \leq 2B$ and $|Z_{n+1}(z)^2 - \nu[U_n](z)| \leq 2B^2$ almost surely. Setting $c_B := 2\max\{B, B^2\}$, it follows that $\|W_{n+1}\|_\infty \leq c_B$ almost surely. Therefore, for any $m > 0$,
\[
    \sup_n \mathbb{E}\sqbk{\|W_{n+1}\|_\infty^m \mid \mathcal{F}_n} \leq c_B^m < \infty.
\]

\subsubsection*{Compute covariance and conclude the CLT}
With the required assumptions verified, the Mokkadem--Pelletier theorem yields the joint convergence
\begin{equation}\label{eq:CLT}
  \begin{bmatrix} \sqrt{\alpha_n^{-1}}\,(U_n - U_{\varepsilon}^{*}) \\[4pt] \sqrt{\beta_n^{-1}}\,\crbk{(m_n, g_n) - (m^{*}_\varepsilon, g^{*}_\varepsilon)} \end{bmatrix}
  \Ra
  \cN\left( 0, \begin{bmatrix} \Sigma_U & 0 \\ 0 & \Sigma_{(m,g)} \end{bmatrix} \right).
\end{equation}
We compute the asymptotic covariance matrices from the ingredients obtained earlier in the proof. First, we recall that with $U$ replacing $\theta$ in Theorem \ref{thm:MP},
\[
    \Gamma_U := \Gamma_{11} + Q_{12} Q_{22}^{-1} \Gamma_{22} (Q_{22}^{-1})^\top Q_{12}^\top - \Gamma_{12} (Q_{22}^{-1})^\top Q_{12}^\top - Q_{12} Q_{22}^{-1} \Gamma_{21}.
\]
Since $\Gamma_{11} = \Gamma_{12} = \Gamma_{21} = 0$, and $Q_{22} = -I_{2d}$, this reduces to
\[
    \Gamma_U = Q_{12}\, \Gamma_{22}\, Q_{12}^\top.
\]
Since $Q_{12} = \begin{bmatrix} A & B \end{bmatrix}$ and the samples are independent across coordinates, $\Gamma_U$ is diagonal with entries, for each $z\in\cZ$,
\[
    \Gamma_U(z,z)
    =
    A(z,z)^2\, V(z) + 2\, A(z,z)\, B(z,z)\, C(z) + B(z,z)^2\, W(z),
\]
and $\Gamma_U(z,z')=0$ for $z\neq z'$, which is exactly \eqref{eq:gammaU}.

The fast Lyapunov equation $Q_{22}\, \Sigma_{(m,g)} + \Sigma_{(m,g)}\, Q_{22}^\top = -\Gamma_{22}$ reduces, by $Q_{22} = -I_{2d}$, to the explicit solution
\[
    \Sigma_{(m,g)} = \frac{1}{2}\,\Gamma_{22}.
\]
The slow Lyapunov equation determines $\Sigma_U$ as the unique solution of
\[
    \crbk{H + \frac{1}{2a}\, I}\Sigma_U + \Sigma_U
    \crbk{H + \frac{1}{2a}\, I}^\top = -\Gamma_U = -Q_{12}\, \Gamma_{22}\, Q_{12}^\top.
\]
which coincides with \eqref{eq:sigmaU}.

This completes the proof of Theorem \ref{thm:convergence-clt}. 
\end{proof}

\section{Proof of Auxiliary Lemmas}
\subsection{Proof of Lemma \ref{lem:robbins-siegmund}}\label{app:proof_of_RS}

\begin{proof}
We show the claim using the Robbins--Siegmund almost-supermartingale convergence theorem in \citet{robbins1971convergence}, which we recall here for completeness.

\begin{theorem}Let $Z_n, \beta_n, \xi_n, \zeta_n$ be nonnegative $\mathcal{F}_n$-measurable sequences such that
\[
  \mathbb{E}[Z_{n+1} \mid \mathcal{F}_n] \leq (1 + \beta_n)\,Z_n + \xi_n - \zeta_n \qquad \text{a.s.\ for all } n,
\]
satisfying $\sum_{n=1}^\infty \xi_n < \infty$ and $\sum_{n=1}^\infty \beta_n < \infty$ a.s. Then $Z_n$ converges a.s.\ to a finite random variable and $\sum_{n=1}^\infty \zeta_n < \infty$ almost surely.
\end{theorem}

We apply this theorem with the identifications
\[
  Z_n := V_n, \qquad \xi_n := b_n, \qquad \beta_n = 0, \qquad \zeta_n := a_n\,V_n.
\]
Since $V_n \geq 0$, we have $\zeta_n \geq 0$. Rewriting~\eqref{eq:RS-condition} gives
\[
  \mathbb{E}[V_{n+1} \mid \mathcal{F}_n] \leq V_n + b_n - a_n\,V_n = Z_n + \xi_n - \zeta_n,
\]
so the conditions of the theorem are satisfied. Moreover, \eqref{eq:RS-summability} implies $\sum_{n=1}^\infty \xi_n = \sum_{n=1}^\infty b_n < \infty$. Hence, the Robbins--Siegmund theorem applies and yields:
\begin{align}
  &V_n \to V_\infty \quad \text{a.s.\ for some finite } V_\infty \geq 0, \label{eq:RS-conv} \\
  &\sum_{n=1}^\infty a_n\,V_n < \infty \quad \text{a.s.} \label{eq:RS-sum}
\end{align}
It remains to show that $V_\infty = 0$ almost surely. 

Let $A\subset\Omega$ with $P(A) = 1$ such that for all $\omega\in A$, \eqref{eq:RS-conv} and \eqref{eq:RS-sum} hold. Suppose that for some $\omega\in A$, $V_\infty(\omega) > 0$. Then there exists $N(\omega)$ such that for all $n \geq N(\omega)$,
$V_n(\omega) \geq \tfrac{1}{2}\,V_\infty(\omega).$ But since $\sum_n a_n = \infty$, we must have that
\begin{align*}
  \sum_{n=1}^\infty a_n\,V_n(\omega) &\geq \sum_{n=N(\omega)}^\infty a_n\,V_n(\omega) \\
  &\geq \frac{1}{2}\,V_\infty(\omega) \sum_{n=N(\omega)}^\infty a_n \\
  &= \infty.
\end{align*}
This and \eqref{eq:RS-sum} implies that $$ 0 =P(\set{V_{\infty} > 0} \cap A) =P(V_{\infty} > 0); $$
i.e. $V_\infty = 0$ a.s. This completes the proof. 
\end{proof}

\subsection{Proof of Lemma \ref{lem:slow_convergence}}\label{app:proof_of_slow_conv}

\begin{proof}

By Theorem \ref{thm:contraction_eps}, $\cR_\varepsilon$ is an $L$-contraction in $\|\cdot\|_\infty$ and 
$U^*_\varepsilon = R_\varepsilon[U^*_\varepsilon]$. So, with the recursion in \eqref{eq:rewrite_U_iter_with_eta} and $\alpha_n\leq1$, we have
\begin{equation}
\begin{aligned}\label{eq:convergeU}
\|U_{n+1} - U^*_\varepsilon\|_\infty
&= \bigl\| U_n - U^*_\varepsilon 
+ \alpha_n \crbk{\cR_\varepsilon[U_n] 
- U_n  + \eta_n } \bigr\|_\infty \\
&= \bigl\| U_n - U^*_\varepsilon 
+ \alpha_n \crbk{\cR_\varepsilon[U_n] - \cR_\varepsilon[U^*_\varepsilon] 
- (U_n - U^*_\varepsilon) + \eta_n } \bigr\|_\infty \\
&\le (1-\alpha_n)\|U_n - U^*_\varepsilon\|_\infty 
+ \alpha_n \|R_\varepsilon[U_n] - R_\varepsilon[U^*_\varepsilon]\|_\infty 
+ \alpha_n \|\eta_n\|_\infty \\
&\le (1-\alpha_n)\|U_n - U^*_\varepsilon\|_\infty 
+ \alpha_n L \|U_n - U^*_\varepsilon\|_\infty 
+ \alpha_n \|\eta_n\|_\infty \\
&= \crbk{1 - (1-L)\alpha_n}\|U_n - U^*_\varepsilon\|_\infty 
+ \alpha_n \|\eta_n\|_\infty.
\end{aligned}
\end{equation}
We define
\[V_n := \|U_n - U^*_\varepsilon\|_\infty\quad \text{and}\quad
a_n := (1 - L)\alpha_n,\] and aim to show that $V_n\ra 0$ a.s.  

Fix any $\iota > 0$. Since $\|\eta_n\|_\infty \to 0$
a.s. by \eqref{eq:tracking-convergence}, for almost every $\omega\in\Omega$, there exists a finite random time $N(\omega)$ such that
$\|\eta_n(\omega)\|_\infty \leq (1-L)\iota$ for all $n \geq N(\omega)$. 

Then, for such $\omega$ and $n\geq N(\omega)$,
$\alpha_n \|\eta_n(\omega)\|_\infty\leq (1-L)\alpha_n \iota = a_n \iota$. Moreover, the inequality \eqref{eq:convergeU} gives
\begin{equation*}
    V_{n+1}(\omega) \leq (1-a_n)V_n(\omega) + a_n\iota
    = (1-a_n)(V_n(\omega) - \iota) + \iota.
\end{equation*}
Since $L < 1$ and $\alpha_n <1$, we have $(1-a_n) \in [0,1]$. Setting $W_n := (V_n - \iota)_+ \geq 0$, we obtain
\begin{align*}
    W_{n+1}(\omega) = (V_{n+1}(\omega) - \iota)_+
    &\leq \crbk{(1-a_n)(V_n(\omega) - \iota)}_+ 
    = (1-a_n)W_n(\omega).
\end{align*}
for all $n \geq N(\omega)$. Iterating the inequality 
$W_{n+1}(\omega) \leq (1-a_n)W_n(\omega)$ forward from $N(\omega)$ gives
\begin{equation}\label{eqn:tu_Wn_bd}
    W_n(\omega) \leq W_{N(\omega)}(\omega) \prod_{k=N(\omega)}^{n-1}(1-a_k).
\end{equation}

Using $1 - x \leq e^{-x}$ for $x \geq 0$, the product satisfies
\begin{equation*}
    \prod_{k=N(\omega)}^{n-1}(1-a_k) 
    \leq \exp\left(-\sum_{k=N(\omega)}^{n-1} a_k\right)
    \leq \exp\left(-\sum_{k=N(\omega)}^{n-1} (1-L)\alpha_k\right).
\end{equation*}
Since $\sum_k \alpha_k = \infty$, the exponent satisfies
\(
    \sum_{k=N(\omega)}^{n-1}(1-L)\alpha_k \to \infty\text{, as } n \to \infty,
\)
so that $\prod_{k=N(\omega)}^{n-1}(1-a_k) \to 0$. 

On the other hand, by the uniform boundedness \eqref{eq:uniform-bounds}, $W_{N(\omega)} \leq V_{N(\omega)}\vee \iota \leq2B \vee \iota < \infty$. So from \eqref{eqn:tu_Wn_bd}, we conclude that $W_n \to 0$ a.s. Therefore
\begin{equation*}
    \limsup_{n \to \infty} V_n \leq \iota \quad \text{a.s.}
\end{equation*}
Since $\iota > 0$ is arbitrary,
$V_n \to 0$ a.s.; i.e., $\|U_n - U_\varepsilon^*\|_\infty \to 0$ a.s.
\end{proof}

\FloatBarrier

\section{Additional Numerical Experiment Details}
\subsection{Algorithm Parameter Specifications}
We choose $\varepsilon = 10^{-6}$ for the stabilized operator $\mathcal{R}_\varepsilon$.  For figure \ref{fig:convergence-clt}, We run the MVSA algorithm with step size parameters
\(
a = 3.0,\ \tau = 0.9, \ b = a^\tau
\), satisfying Assumption \ref{ass:step-sizes}.

\subsection{Optimal Robust Q-Values} Figure~\ref{fig:Ustar} reports the 
fixed-point $U^*_\varepsilon(s,a)$ computed by value iteration on the inventory MDP with $\gamma = 0.7$, $\delta = 0.1$, $\varepsilon = 10^{-6}$. The left panel displays the full Q-value table as a heatmap, with red markers indicating the greedy action $a^*(s) = \arg\max_a U^*_\varepsilon(s,a)$  at each state. 

The right panel plots the induced optimal value function  $v^*(s) = \max_a U^*_\varepsilon(s,a)$ with the optimal action annotated. The greedy policy recovers the familiar base-stock structure: $a^*(s) = 0$ for sufficiently large inventory levels, with order quantities increasing as $s$ decreases.
\begin{figure}[ht]
    \centering
    \includegraphics[width=1\textwidth]{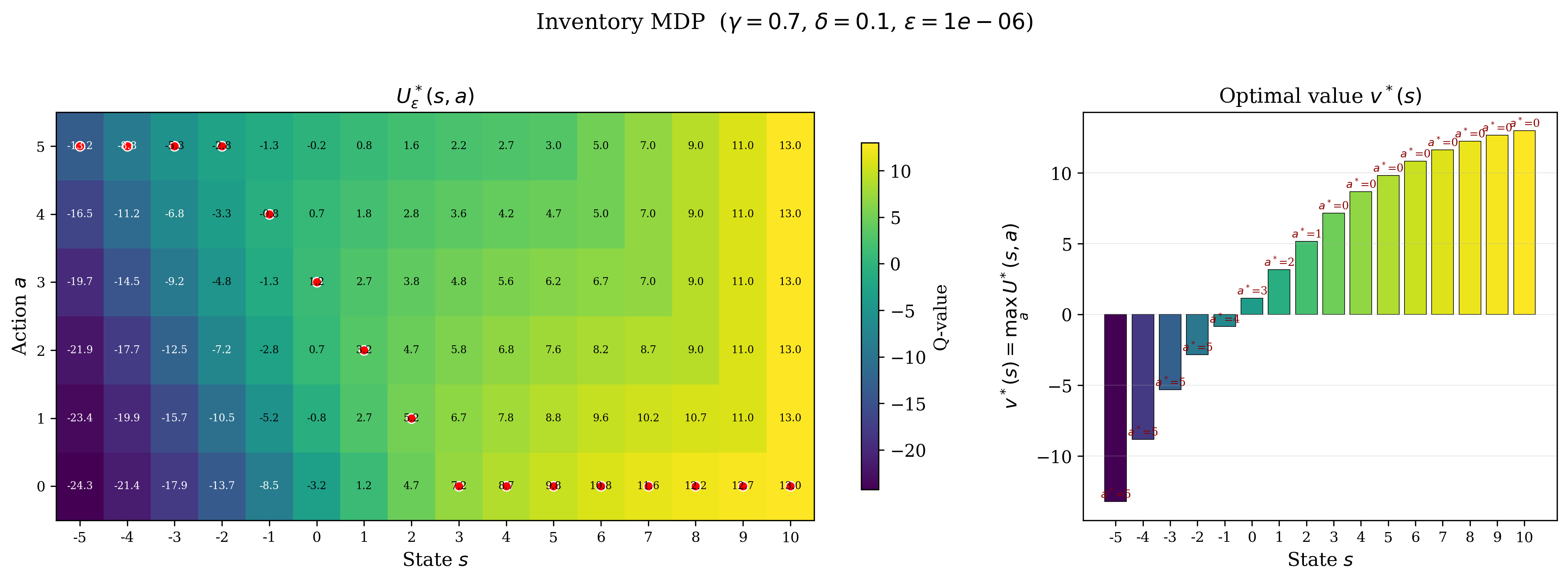}
    \caption{({Left}) Heatmap of $U^*_\varepsilon(s,a)$; red markers indicate the greedy action $a^*(s)$. ({Right}) Optimal value function $v^*(s) = \max_a U^*_\varepsilon(s,a)$ with optimal actions annotated.}
    \label{fig:Ustar}
\end{figure}

\end{document}